\documentclass{siamart190516} 

\usepackage[sort]{cite}

\usepackage{enumitem,amssymb,amsmath,bm,bbm}
\usepackage{amsfonts}
\usepackage{graphicx,subcaption}
\usepackage{algorithmic}
\usepackage{mathdots}
\usepackage{calc} 
\usepackage{multirow}
\usepackage{array}
\newcolumntype{?}{!{\vrule width 1pt}}

\usepackage{tikz,ifthen}

\ifpdf
  \DeclareGraphicsExtensions{.eps,.pdf,.png,.jpg}
\else
  \DeclareGraphicsExtensions{.eps}
\fi

\newcommand{\A}{{A}}
\newcommand{\B}{{B}}

\renewcommand{\L}{{\Lambda}}
\newcommand{\X}{{X}}
\newcommand{\V}{{V}}

\newcommand{\U}{{U}}
\newcommand{\Q}{{Q}}
\newcommand{\T}{{T}}

\newcommand{\CC}{\mathbb{C}}
\newcommand{\RR}{\mathbb{R}}
\newcommand{\ZZ}{\mathbb{Z}}
\newcommand{\NN}{\mathbb{N}}
\newcommand{\VS}[1]{\mathrm{#1}}
\newcommand{\Id}{I}
\newcommand{\vb}[1]{{\bm{#1}}}

\newsiamremark{remark}{Remark}
\newsiamremark{hypothesis}{Hypothesis}
\crefname{hypothesis}{Hypothesis}{Hypotheses}
\newsiamthm{claim}{Claim}
\newsiamremark{example}{Example} 

\headers{Group-invariant tensor train networks}{B. Sprangers and N. Vannieuwenhoven}

\title{Group-invariant tensor train networks for supervised learning\thanks{Submitted to the editors DATE.}} 

\author{Brent Sprangers\thanks{KU Leuven, Department of Computer Science, NUMA division, B-3000 Leuven, Belgium
  (\email{brent\_sprangers@outlook.com}).}
\and Nick Vannieuwenhoven\thanks{KU Leuven, Department of Computer Science, NUMA division, B-3000 Leuven, Belgium (\email{nick.vannieuwenhoven@kuleuven.be}); Leuven.AI -- KU Leuven Institute for AI, B-3000 Leuven, Belgium}.}

\usepackage{amsopn}

\usepackage{booktabs}

\newcommand{\tensor}[1]{\bm{\mathcal{#1}}}
\newcommand{\matr}[1]{#1}

\allowdisplaybreaks

\begin{document}
 
\maketitle

\begin{abstract}
Invariance has recently proven to be a powerful inductive bias in machine learning models. One such class of predictive or generative models are tensor networks. We introduce a new numerical algorithm to construct a basis of tensors that are invariant under the action of normal matrix representations of an arbitrary discrete group. This method can be up to several orders of magnitude faster than previous approaches. The group-invariant tensors are then combined into a group-invariant tensor train network, which can be used as a supervised machine learning model. We applied this model to a protein binding classification problem, taking into account problem-specific invariances, and obtained prediction accuracy in line with state-of-the-art deep learning approaches.
\end{abstract}



\section{Introduction}\label{sec_introduction}
The concept of \textit{equivariance} asserts that when the input to a function changes in some specific way, then the output of that function changes in a correspondingly predictable way. \textit{Invariance} is a special case wherein the function's output does not change under specific input changes.
For example, a function that takes a square picture as input and outputs whether the input contains a cat is invariant under reflections along a bisector and rotations by multiples of $\frac{\pi}{2}$ radians. A function taking square pictures as input and outputting the tightest bounding box around all cats in the input is equivariant under those same reflections and rotations.
Equivariance has long been an important property of convolutional neural networks without being noticed as such. However, since Cohen and Welling \cite{cohen2016group} introduced group-equivariant convolutional networks, research in exploiting equivariance in neural networks has boomed; an overview of works involving equivariance in machine learning can be found in \cite{awesome}. Besides providing better statistical efficiency for learning, equivariance is a powerful inductive bias, and models respecting problem-specific equivariances tend to generalize better \cite{bronstein2021geometric}. 

While neural networks are an omnipresent machine learning model, recent years have seen the popularization of another type of network: \textit{Tensor networks}. Recall that a tensor is a higher-order generalization of matrices and vectors, which after choice of basis can be identified with a multidimensional array. By \textit{order} we mean the number of indices. Tensor networks originated in the physics community as a way to efficiently represent quantum wave functions that live in a vector space whose dimension scales exponentially with the order of the tensor by means of polynomially-scaling resources \cite{Bridgeman_2017, 2011dmrg}. 
Conceptually, a tensor network is a computational graph in which the vertices represent tensors (referred to as \emph{cores}) and labeled edges between tensors indicate the \emph{contraction} of the two tensor cores along a pair of dual vector spaces that constitutes the label of the edge \cite{YL2018}. The tensor represented by a tensor network is the one obtained from evaluating the contractions described by this graph. With a properly chosen graph, a tensor network can represent an interesting subset of high-order tensors by means of low-order tensors (typically orders $2$, $3$ or $4$). 
Examples of commonly used tensor network architectures include \emph{matrix product states} \cite{2011dmrg} or \emph{tensor train networks} (TTN) \cite{oseledets}, and \emph{tree tensor networks} or the \textit{hierarchical Tucker decomposition} \cite{hackbusch,Grasedyck2010}, among others.
The focus of this paper are TTNs, which have a chain-like structure of order-$3$ tensors, as visualized in \cref{fig:tt_ml}.

Stoudenmire and Schwab \cite{stoudenmire2017supervised} showed how tensor networks can be used as discriminative machine learning models, and ever since they have been applied to machine learning tasks such as generative modelling \cite{Cheng_2019}, natural language processing \cite{pestun2017language, tangpanitanon2021explainable}, and image segmentation \cite{selvan2021segmenting}. 
Their use is similar to kernel methods. 
The input data is first mapped to an exponentially large space with a \emph{feature map} $\Phi(\vb{x})$ that consists of taking the tensor product $\otimes$ of \emph{local feature maps} ${\phi}_i: \RR \rightarrow \RR^{n_i}$ so that
\begin{equation}\label{eq:tensorfeature}
\Phi(\vb{x}) = {\phi}_1(x_1) \otimes \dots \otimes {\phi}_k(x_k) \in \RR^{n_1}\otimes \dots \otimes \RR^{n_k} \simeq \RR^{n_1 \times \cdots \times n_k}, 
\end{equation}
where $x_i$ denotes the $i$th element of $\vb{x} \in \RR^k$, as proposed in \cite{novikov2017exponential,KNO2018}.
Thereafter, the kernelized input is provided to a \textit{feature tensor} $\tensor{F}$ in $(\RR^{n_1}\otimes\dots\otimes\RR^{n_k})^* \otimes \RR^{n_{k+1}}$, which represents a linear map from the input space $\RR^{n_1}\otimes\dots\otimes\RR^{n_k}$ to $\RR^{n_{k+1}}$. To keep this approach computationally tractable, additional structure, such as a TTN, must be imposed on the feature tensor (compare \cref{fig:tn_ml_gen,fig:tt_ml}), which comprises the \textit{kernel trick} in this approach.
The output vectors can then be used in standard machine learning pipelines.

\begin{figure}[tb]
	\centering
	\caption{\footnotesize Illustration of how tensor networks can be used as machine learning models in graphical notation \cite{Bridgeman_2017}. The grey nodes denote the rank-1 input tensor $\Phi(\vb{x})$ formed by the tensor product of the local feature maps $\vb{\phi}_i(x_i)$. The contraction of $\Phi(\vb{x})$ with the \textit{feature tensor} (the rectangles with rounded corners) yields the feature vector (the single output edge at the top of both graphs). In the left figure, the feature tensor is a single order-$6$ tensor ($5$ input spaces and one output space). In the right figure, the feature tensor has a TTN structure and is represented by two matrices (the leftmost and rightmost rectangles), an order-$4$ tensor (middle rectangle), and two order-$3$ tensors.}
	\label{fig:tn_ml}
	\begin{subfigure}[t]{0.4\textwidth}
		\centering
\begin{tikzpicture}[scale=0.8]
\draw[very thick] (2.5,0) -- ++(0,2);
\foreach \i in {0,1,2,3,4} {
\draw[very thick] (1.25*\i,0) -- ++(0,1);
\draw[very thick,fill=black!10] (1.25*\i,0) circle (0.4);
}
\draw[very thick,fill=white,rounded corners] (-.4,0.6) rectangle ++(5.8,1);
\end{tikzpicture}

		\caption{\footnotesize Model with general tensor. This is intractable for large inputs.}
		\label{fig:tn_ml_gen}
	\end{subfigure}
	\hfill
	\begin{subfigure}[t]{0.4\textwidth}
		\centering
		\begin{tikzpicture}[scale=0.8]
\draw[very thick] (2.5,0) -- ++(0,2);
\draw[very thick] (0,1.1) -- ++(5,0);
\foreach \i in {0,1,2,3,4} {
\draw[very thick] (1.25*\i,0) -- ++(0,1);
\draw[very thick,fill=black!10] (1.25*\i,0) circle (0.4);
\draw[very thick,fill=white,rounded corners] (1.25*\i-.5,0.6) rectangle ++(1,1);
}
\end{tikzpicture}
		\caption{\footnotesize Model with TTN. Computation and memory scales linearly in the input.}
		\label{fig:tt_ml}
	\end{subfigure}
\end{figure}
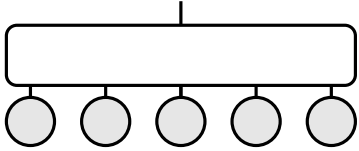

The central question of this paper is how to efficiently and algorithmically impose problem-specific equivariance and invariance conditions on tensors, generally, and TTN models, specifically. The aim is to imbue these machine learning models with a problem-specific inductive bias that allows them to more accurately generalize from the training data \cite{bronstein2021geometric}; see \cref{sec:experiments} below for an illustration of this effect.

Mathematically, the essence of the problem we study is as follows. A feature tensor $\tensor{F}$ bijectively corresponds to a \textit{multilinear map} $f$ \cite[Chapter 1]{Greub1978}. Such maps take inputs in several vector spaces and send them to another vector space. We say that $f$ is equivariant under the joint action of the invertible linear maps $M_g^{i} : \VS{W}_i \to \VS{W}_i$, $g=1,\ldots,s$, $i=1,\ldots,k+1$, if simultaneously the conditions
\begin{equation}\label{eqn_def_invariance}
  f( M_g^{1} \vb{x}_1, \ldots, M_g^{k} \vb{x}_{k} ) = M_g^{k+1} f(\vb{x}_1, \ldots, \vb{x}_k), \quad g=1,\ldots,s,
\end{equation}
hold for all $\vb{x}_i\in\VS{W}_i$.\footnote{Unless explicitly mentioned, all superscripts in this paper refer to indices. They should not be interpreted as powers.}
When all $M_g^{k+1}$ equal the identity map $\Id$ on $\VS{W}_{k+1}$, then the above conditions define invariance. These natural definitions for the multilinear map $f$ then induce a notion of equivariance and invariance on the corresponding tensor $\tensor{F}$. We will see that they both correspond to the same fundamental concept in representation theory, namely the \textit{invariance of $\tensor{F}$ under a group action} \cite[Section 1.1]{FH2004}. Because of this common group structure, one can actually simultaneously impose equivariance conditions ($M_g^{k+1} \ne \Id$) and invariance conditions ($M_g^{k+1}=\Id$) on $f$.
These invariant tensors $\tensor{F}$ span a linear subspace, called a \textit{subrepresentation} \cite[Chapter 1]{FH2004}. The main original contribution of this paper is an efficient numerical algorithm for computing an orthonormal basis for this vector space of invariant tensors when the \textit{representations} $M_g^{i}$ are normal. As a general group-invariant tensor is a linear combination of these basis vectors, optimization over the set of group-invariant tensors reduces to unconstrained optimization over the coefficients with respect to this basis. As we will see in \cref{sec:pval,sec:experiments}, this can significantly reduce the dimension of the search space, reducing memory consumption and computational time.

Prior work in the literature, discussed in \cref{sec:rel_work} below, has already investigated the idea of imposing the conditions \cref{eqn_def_invariance} on the associated tensor $\tensor{F}$. The main focus, however, has been on theoretically deriving a basis of the subrepresentation in specific cases.
In contrast to the existing literature, in this work, we propose a numerical algorithm to combine group invariance and TTNs efficiently, at least for invariances originating from finite group actions.
Our contributions are twofold. First, we introduce a new algorithm to construct group-invariant tensors that can be orders of magnitude more efficient than the previous state-of-the-art method by Finzi, Welling and Wilson \cite{finzi2021practical}, discussed in \cref{sec:rel_work}. Second, we employ our proposed method to construct group-invariant TTNs that we apply as a machine learning model for classification tasks. Our numerical experiments on a supervised learning problem regarding the binding of transcription factors to DNA sequences indicate that the proposed method is competitive with respect to state-of-the-art invariant neural networks in terms of area under the ROC curve.

The paper is organized as follows. 
\Cref{sec:background} introduces background material on representation theory of finite groups, multilinear functions, and tensors, and shows how the natural definition of equivariance of a multilinear function implies that the corresponding tensor is invariant under a group action in the sense of representation theory. \Cref{sec:tnequiv} recalls the mathematical definition of TTNs and derives our novel method to construct $G$-invariant TTNs. A key ingredient of our approach concerns the choice of the first generator of the group for which the TTN is invariant; a good heuristic is presented in \cref{sec:pval} for cyclic, dihedral, symmetric, and dicyclic groups and their products. 
Before presenting our experimental results, \cref{sec:rel_work} briefly recalls the alternative, state-of-the-art group-invariant tensor basis construction method of \cite{finzi2021practical}.
In \cref{sec:experiments}, the practical performance of our new group-invariant tensor basis construction is compared with the state of the art, and the group-invariant TTNs are compared with neural and tensor networks in two benchmark problems. In particular, we compare with a recent deep learning method on a protein binding classification problem and attain state-of-the-art performance with a group-invariant TTN that takes certain application-specific invariances into account. Finally, \cref{sec:conclusions} presents our conclusions.

\subsection*{Notation}
To lighten the notation, all positive superscripts in this paper refer to indices rather than powers, unless explicitly mentioned otherwise.
Scalars will be denoted by lowercase letters ($d$), vectors by boldface lowercase letters ($\vb{v}$), matrices by uppercase letters ($\A$), and tensors by boldface calligraphic letters ($\tensor{T}$). The tensor product is denoted by $\otimes$. Throughout the paper, we let 
$\VS{W}_i$ be real vector spaces of dimension $n_i$. 
The dual vector space of $\VS{W}_i$ is denoted by $\VS{W}_i^*$. 
We also define the following tensor product space
\[
\VS{T}_{k,d} = \VS{W}^*_{1} \otimes \cdots \otimes \VS{W}^*_k \otimes \VS{W}_{k+1}  \otimes \cdots \otimes \VS{W}_d,
\]
where $k, d \in \NN$; an empty product of dual or regular vector spaces should be interpreted as $\RR$. The space of linear maps from $\VS{V}$ to $\VS{W}$ is denoted by $L(\VS{V}; \VS{W})$.
The identity matrix and map are both denoted by $\Id$; a subscript is sometimes added to indicate the size of this matrix or the vector space on which it is the identity. 
The transpose of a linear map $L : \VS{W}_1 \to \VS{W}_0$ is $L^* : \VS{W}_0^* \to \VS{W}_1^*$. The transpose of a matrix $\A \in \CC^{m \times n}$ is denoted by $\A^\top$ and its conjugate transpose is denoted by $\A^H$.

\subsection*{Acknowledgments}
Part of the resources and services used in this work were provided by the VSC (Flemish Supercomputer Center), funded by the Research Foundation---Flanders (FWO) and the Flemish Government. This work was partially supported by KU Leuven Internal Funds grant STG/19/002.

We thank the reviewers for their detailed reading and their suggestions that spurred us to improve this paper, notably in \cref{sec_introduction,sec:background,sec:tnequiv}, and we thank editor Afonso Bandeira for handling the process.

\section{Group-invariant tensors}\label{sec:background}
This section explains how group-invariant tensors arise from multilinear maps that satisfy the condition in \cref{eqn_def_invariance}. We first recall basic representation theory of finite groups. Then, we show that imposing the conditions \cref{eqn_def_invariance} is equivalent to imposing that the map $f$ is invariant under the action of an associated finite group generated by these conditions. Thereafter, the relation between multilinear maps and tensor is recalled. Finally, we show that the group-invariance of a multilinear map $f$ translates into a natural invariance condition on the tensor representing that map.

\subsection{Representation of finite groups}
We recall the basic representation theory of finite groups that we need in this paper to define group-invariant tensors; for full details see \cite[Chapter 1]{FH2004}.

A \textit{group} $(G, *)$ is a set $G$ equipped with a binary operation
\(*: G\times G \rightarrow G\)
for which the following properties hold:
\begin{enumerate}
	\item \emph{Associativity:} $\forall a,b,c  \in G:\; (a*b)*c = a*(b*c)$;
	\item \emph{Neutral element:} $\exists e \in G,\; \forall a \in G: \; a*e = e*a = a$;
	\item \emph{Inverses:} $\forall a \in G, \;\exists b \in G:\; a*b = b*a = e$.
\end{enumerate}
\begin{example}
	The set of integers $\{0, 1\}$ with the addition modulo 2, $(\ZZ_2, +)$, forms a group. This group is often called the \emph{parity group}.
\end{example}

Somewhat related to a basis for a linear space, are the \textit{generators} of a group. If all the elements of a group $(G, *)$ can be constructed from the composition of elements of $\{g_1,\dots, g_s\} \subset G$ and their inverses, the group is said to be generated by $\{g_1,\dots, g_s\}$. This is usually denoted by $G = \langle g_1,\dots, g_s \rangle$. A group that is generated by a single element is called a \emph{cyclic group}, i.e.,  $G = \langle g_1\rangle$ is cyclic.

An $n$-dimensional \textit{group representation} $\rho$ of a group $G$ is a \textit{group-homomorphism} from $G$ to the group of linear automorphisms $\mathrm{Aut}(\VS{W})$ of an $n$-dimensional vector space $\VS{W}$ with composition of maps as binary operation \cite[Chapter XVIII, section 1]{Lang}. More concretely, after choosing an (orthonormal) basis for $\VS{W}$, each automorphism can be represented by an $n \times n$ invertible matrix in the \textit{general linear group} $\mathrm{GL}(n)$ (with matrix multiplication as operation), so that we can identify $\rho$ with 
\begin{equation*}
\rho: G \rightarrow \text{GL}(n),\quad g \mapsto \U_g,
\end{equation*}
so the homomorphism condition
\(
\U_{g*h} = \rho(g*h) = \rho(g) \cdot \rho(h) = \U_g \cdot \U_h
\)
is satisfied.

Every group has at least one representation on $\VS{W}$. Indeed, the \textit{trivial representation} \cite[Chapter XVIII, section 1]{Lang} sets $\rho(g) = \Id_n$ for every $g \in G$, which clearly satisfies the required homomorphism condition.

We call a representation \emph{normal} if, after choosing an orthonormal basis of $\VS{W}$ to identify $\mathrm{Aut}(\VS{W})$ and $\mathrm{GL}(n)$, the image of $\rho$ is contained in the set of normal matrices. A representation is orthogonal if the image is contained in the orthogonal group. Orthogonal representations, which include trivial representations, are normal. All of the eigenvalues of normal representations are roots of unity.

\begin{example}\label{eq:rep_parity}
	Consider the parity group introduced earlier. An orthogonal representation, which is also normal, of $\rho: \ZZ_2 \rightarrow \text{GL}(2, \RR)$ is given by:
	\begin{align}
	\rho(0) &= \begin{pmatrix}
	1 & 0 \\
	0 & 1 
	\end{pmatrix},
	&
	\rho(1) &= \begin{pmatrix}
	0 & 1 \\
	1 & 0
	\end{pmatrix}.
	\end{align}
\end{example}

A representation $\rho : G \to \mathrm{Aut}(\VS{W})$ induces a \textit{dual representation} $\rho^* : G \to \mathrm{Aut}(\VS{W}^*)$ \cite[Chapter XVIII, section 1]{Lang}. When representing the automorphisms with respect to the dual basis of $\VS{W}$, this dual representation satisfies 
\begin{align}\label{eqn_dual_representation}
 \rho^* : G \to \mathrm{GL}(n), \quad g\mapsto \U_g^{-\top}
\end{align}
with $\U_{g*h}^{-\top} = \rho^*(g*h) = \rho^*(g) \cdot \rho^*(h) = \U_g^{-\top} \cdot \U_h^{-\top}$. For normal representations, the dual representation is also normal. 
For orthogonal representations, the dual representation coincides with the original representation: $\rho^* = \rho$. 

\subsection{Group-invariant maps}\label{sec_sub_groupstructure}
Next, we show that for an arbitrary map $f$ satisfying \cref{eqn_def_invariance} there exists an underlying group $G$ and representations $\rho_i : G \to \mathrm{GL}(n_i)$ so that for all $j=1,\ldots,s$ we have $M_{j}^{i} = \rho_i(g_j)$, where $g_j \in G$.

Consider \cref{eqn_def_invariance} for an arbitrary map $f : \VS{W}_1\times\cdots\times\VS{W}_{k}\to\VS{W}_{k+1}$ and let $d=k+1$. 
We can identify each individual condition by a tuple of invertible linear maps 
\begin{equation}\label{eqn_tuple}
 M_g = ( M_{g}^{1},\ldots,M_g^{d} ), \quad\text{where } g=1,\ldots,s.
\end{equation}
We say that $f$ is invariant under $M_g$.
Let $G$ denote the set of \textit{all} tuples of linear maps under which $f$ is invariant. We have $M_1, \ldots, M_s \in G$ by definition, and $\Id := (\Id,\ldots,\Id) \in G$ because the corresponding invariance condition is trivial.
We observe that if $M_g = (M_g^{1},\ldots,M_g^{d}) \in G$ and $M_{h} = (M_{h}^{1},\ldots,M_{h}^{d}) \in G$, then 
\begin{align*}
 f(M_{g}^{1} M_{h}^{1} \vb{x}_1, \ldots, M_{g}^{k} M_{h}^{k} \vb{x}_{k}) 
 = M_{g}^{d} f(M_{h}^{1} \vb{x}_1, \ldots, M_{h}^{k} \vb{x}_{k})
 = M_{g}^{d} M_{h}^{d} f(\vb{x}_1, \ldots, \vb{x}_{k}).
\end{align*}
Hence $M_{g}*M_{h} := (M_{g}^{1} M_{h}^{1},\ldots,M_{g}^{d} M_{h}^{d}) \in G$.  
Finally, since $M_g^{i}$ is invertible and \cref{eqn_def_invariance} applies for all $\vb{x}_i \in \VS{W}_i$, plugging $\vb{y}_i = M_g^{i} \vb{x}_i$ into \cref{eqn_def_invariance} yields
\begin{align*}
f( (M_g^{1})^{-1} \vb{y}_1, \ldots, (M_g^{k})^{-1} \vb{y}_{k} ) 
=
f(\vb{x}_1, \ldots, \vb{x}_k)
&= (M_g^{d})^{-1} f( M_g^{1} \vb{x}_1, \ldots, M_g^{k} \vb{x}_{k} )\\
&= (M_g^{d})^{-1} f( \vb{y}_1, \ldots, \vb{y}_{k} ).
\end{align*}
Hence $M_{g}^{-1} := (  (M_g^{1})^{-1}, \ldots, (M_g^{d})^{-1} ) \in G$, and, evidently, $M_g^{-1}*M_g = M_g* M_g^{-1}=I$. This proves that $(G,*)$ is a finitely generated group.
Moreover, it is verified that $\rho_i : G \to \mathrm{Aut}(\VS{W}_i), (M^{1},\ldots,M^{d}) \mapsto M^{i}$ is a representation of $G$ on $\VS{W}_i$.

The above construction shows that, \textit{starting from any finite set of conditions as in \cref{eqn_def_invariance}, an arbitrary map $f$ will be invariant under (the chosen representations of) the whole finite group $G$ generated by $M_1,\ldots,M_s$ from \cref{eqn_tuple}.} For this reason, $f$ is called \textit{$G$-invariant} in representation theory \cite[Chapter 1]{FH2004}.
This terminology may appear to be a misnomer because the invariance is not characterized by the group, but rather by the chosen representation of the group. Nevertheless, this is the standard terminology in representation theory; see \cite[Chapter XVIII, section 1]{Lang} and \cite[Chapter 1]{FH2004}.

\subsection{From multilinear maps to tensors}
Next, we recall the correspondence between multilinear maps and tensors. The \textit{universal property} \cite[Section 1.4]{Greub1978} of the tensor product states that for every multilinear map $f : \VS{W}_1 \times\cdots\times\VS{W}_k \to \VS{W}_{k+1}$ there exists a unique linear map $F : \VS{W}_1 \otimes\cdots\otimes \VS{W}_{k} \to \VS{W}_{k+1}$ such that 
\begin{equation}\label{def_1flattening}
 f(\vb{x}_1,\ldots,\vb{x}_{k}) = F (\vb{x}_1 \otimes \cdots \otimes \vb{x}_k)
\end{equation}
for all $\vb{x}_i\in\VS{W}_i$. This linear map $F$ can be interpreted as a tensor $\tensor{F} \in \VS{W}_1^* \otimes \cdots \otimes \VS{W}_{k}^* \otimes \VS{W}_{k+1} = \mathrm{T}_{k,k+1}$ \cite[Chapter 1]{Greub1978}; note that the first $k$ spaces are dual vector spaces because the linear map expects inputs from those spaces. It can be shown, by using the multilinearity of $f$ and the tensor product, that this tensor is given explicitly by
\[
 \tensor{F} 
 = \sum_{j_1=1}^{n_1} \cdots \sum_{j_{k}=1}^{n_{k}} \vb{e}_{j_1}^{1*}\otimes\cdots\otimes\vb{e}_{j_{k}}^{k*} \otimes f(\vb{e}_{j_1}^{1}, \ldots, \vb{e}_{j_{k}}^{k}),
\]
where $\vb{e}_1^{i}, \ldots, \vb{e}_{n_i}^{i}$ forms an orthonormal basis of $\VS{W}_i$, and $\vb{e}^{i*}_j$ denotes the dual basis vector of $\vb{e}^{i}_j$. Like a matrix that represents a linear map, the foregoing expression simply encodes that the tensor $\tensor{F} \in \mathrm{T}_{k,k+1}$ that represents a multilinear map $f$ is given in coordinates by the evaluation of $f$ on the basis vectors $\vb{e}_{j_1}^{1}\otimes\cdots\otimes\vb{e}_{j_k}^{k}$. Note that this uses that the tensor product of the bases of the individual vector spaces forms a basis for the tensor product space \cite{Greub1978}.

Since the tensor product $\VS{W}_{k+1}\otimes\cdots\otimes\VS{W}_d$ is itself a vector space, in general a $(k,d)$-tensor $\tensor{T}$ is defined to be an element of $\VS{T}_{k,d} = \VS{W}_{1}^*\otimes\cdots\otimes\VS{W}_k^* \otimes \VS{W}_{k+1}\otimes\cdots\otimes\VS{W}_d$. 
Such a tensor can always be written as
\begin{equation*}
\tensor{T} = \sum_{j_{1}=1}^{n_{1}} \cdots \sum_{j_d=1}^{n_d}\mathcal{T}_{j_{1},\dots,j_d}\; \vb{e}_{j_{1}}^{1*} \otimes \cdots \otimes \vb{e}_{j_k}^{k*} \otimes \vb{e}_{j_{k+1}}^{k+1} \otimes \cdots \otimes \vb{e}_{j_d}^{d},
\end{equation*}
where $\vb{e}_j^{i} \in \VS{W}_i$ and $\vb{e}_j^{i*} \in \VS{W}_i^*$ are as above.

Let the tensor $\tensor{T} \in \VS{T}_{k,d}$, and assume we have linear maps $U_i^* : \VS{W}_i^* \to \VS{V}_i^*$, $i=1,\ldots,k$, and linear maps $U_i : \VS{W}_i \to \VS{V}_i$, $i=k+1,\ldots,d$. The \textit{tensor product} of these linear maps $U_{1}^*\otimes\dots\otimes U_k^* \otimes U_{k+1} \otimes\dots\otimes U_d$ is a linear map that lives in 
\[
 L( \VS{W}_{1}^*\otimes\dots\otimes\VS{W}_k^* \otimes \VS{W}_{k+1}\otimes\dots\otimes\VS{W}_d; \VS{V}_{1}^*\otimes\dots\otimes\VS{V}_k^* \otimes \VS{V}_{k+1}\otimes\dots\otimes\VS{V}_d ),
\]
where $L(\VS{W}; \VS{V})$ is the vector space of linear maps from $\VS{W}$ to $\VS{V}$; see Greub \cite[Chapter 1]{Greub1978}. This map is defined as 
\begin{multline*}
(U_{1}^*\otimes\cdots\otimes U_k^* \otimes U_{k+1} \otimes\dots\otimes U_d)(\tensor{T})\\
:= \sum_{j_{1}=1}^{n_{1}} \cdots \sum_{j_d=1}^{n_d} \mathcal{T}_{j_{1},\dots,j_d} (U_{1}^{*} \vb{e}_{j_{1}}^{1*} ) \otimes\dots\otimes (U_k^{*} \vb{e}_{j_k}^{k*} ) \otimes (U_{k+1} \vb{e}_{j_{k+1}}^{k+1}) \otimes\dots\otimes (U_d \vb{e}_{j_d}^{d}).
\end{multline*}
For brevity, it is common in numerical multilinear algebra to write
\begin{align}\label{eqn_multilinear_multiplication}
(U_{1}^*,\ldots,U_k^*, U_{k+1},\dots,U_d) \cdot\tensor{T} := (U_{1}^*\otimes\cdots\otimes U_k^* \otimes U_{k+1} \otimes\dots\otimes U_d)(\tensor{T}).
\end{align}
Evaluating this linear map at $\tensor{T}$ is often referred to as a \textit{multilinear multiplication}.

\subsection{Group-invariant multilinear maps and tensors}
The previous subsection recalled that any multilinear map $f$ corresponds to a tensor $\tensor{F}$. The following result shows how the invariance conditions \cref{eqn_def_invariance} manifest themselves in the tensor $\tensor{F}$.

\begin{proposition}\label{prop_invariance}
Let $f : \VS{W}_{1}\times\cdots\times\VS{W}_k \to \VS{W}_{k+1}$ be a multilinear map, $\tensor{F}\in\VS{W}_1^*\otimes\cdots\otimes\VS{W}_k^*\otimes\VS{W}_{k+1}$ the associated tensor, $G=\langle g_1,\ldots,g_s\rangle$ a finitely-generated group, and $\rho_i : G \to \mathrm{Aut}(\VS{W}_i)$ representations. Then, $f$ is $G$-invariant if and only if 
\begin{equation}\label{eqn_tensorf_invar}
 \tensor{F} = ( \rho_1^*(g),\ldots,\rho_k^*(g), \rho_{k+1}(g) ) \cdot \tensor{F}, \quad \forall g \in \{g_1,\ldots,g_s\},
\end{equation}
where $\rho^*(g) = \rho^{-\top}(g)$ is the dual representation.
\end{proposition}
\begin{proof} 
It follows from \cref{sec_sub_groupstructure} that if \cref{eqn_tensorf_invar} holds for the generators of $G$, then \cref{eqn_tensorf_invar} holds for all group elements $g \in G$. The converse is obvious. 
Hence, in the remainder of the proof, we let $g\in G$ be arbitrary.

From \cref{def_1flattening} we find on the one hand that
\begin{align*}
 f( \rho_1(g) \vb{x}_1, \ldots, \rho_k(g) \vb{x}_k )
 &= F \left( ( \rho_1(g) \vb{x}_1 ) \otimes \cdots\otimes (\rho_k(g) \vb{x}_k) \right)\\
 &= \left( F \circ (\rho_1(g) \otimes \cdots\otimes \rho_k(g)) \right)\left(\vb{x}_1\otimes\cdots\otimes \vb{x}_k \right),
\end{align*}
while on the other hand,
\[
 f( \rho_1(g) \vb{x}_1, \ldots, \rho_k(g) \vb{x}_k ) 
 = \rho_{k+1}(g) f(\vb{x}_1,\ldots,\vb{x}_k) 
 = \left( \rho_{k+1}(g) \circ F \right) (\vb{x}_1\otimes\cdots\otimes \vb{x}_k).
\]
As $F$ is a linear map, both equations hold for arbitrary $\vb{x}_i$, and $\{\vb{e}_{j_1}^{1}\otimes\cdots\otimes\vb{e}_{j_k}^{k}\}_{j_1,\ldots,j_k}$ is a basis of the domain of $F$, we obtain the following equality of linear maps
\[
 F = \rho_{k+1}(g) \circ F \circ ( \rho_1(g) \otimes\cdots\otimes \rho_k(g) )^{-1}.
\]
From this it follows that 
\begin{align} \label{eqn_invar_proof1}
 &F(\vb{x}_1\otimes\cdots\otimes\vb{x}_k)\\
\nonumber &= \rho_{k+1}(g) \circ F (\rho_1^{-1}(g) \vb{x}_1 \otimes\cdots\otimes \rho_k(g)^{-1} \vb{x}_k ) \\
\nonumber &=
 \rho_{k+1}(g) \sum_{j_1=1}^{n_1} \cdots \sum_{j_{k}=1}^{n_{k}} ( \vb{e}_{j_1}^{1*} \rho_1^{-1}(g) \vb{x}_1)   \otimes\cdots\otimes (\vb{e}_{j_{k}}^{k*} \rho_k^{-1}(g) \vb{x}_k ) \otimes f(\vb{e}_{j_1}^{1}, \ldots, \vb{e}_{j_{k}}^{k}).
\end{align}
Note that the last step used a standard identification \cite[Chapter 1]{Greub1978}, namely
\[
\tensor{F} \in \VS{W}_1^* \otimes\cdots\otimes \VS{W}_k^* \otimes \VS{W}_{k+1}
\overset{\cdot_{(k+1)}}{\simeq} L(\VS{W}_1 \otimes\cdots\otimes \VS{W}_k; \VS{W}_{k+1}) \ni \tensor{F}_{(k+1)} := F;
\]
in numerical multilinear algebra identifying $\tensor{F}$ with $\tensor{F}_{(k+1)}:=F$ is called a \textit{flattening} or \textit{matricization}.
Recall that the transpose of a linear map is defined to satisfy
\[
\vb{e}_{j}^{i*} (\rho_i^{-1}(g) \vb{x}_i) 
= (\rho_i^{-\top}(g) \vb{e}_j^{i*})(\vb{x}_i). 
\]
Writing $f(\vb{e}_{j_1}^{1},\ldots,\vb{e}_{j_k}^{k}) = \sum_{j_{k+1}=1}^{n_{k+1}} \tensor{F}_{j_1,\ldots,j_{k+1}} \vb{e}_{j_{k+1}}^{k+1}$ and plugging it into \cref{eqn_invar_proof1} yields
\begin{align*}
 &\tensor{F}_{(k+1)}(\vb{x}_1\otimes\cdots\otimes\vb{x}_k)\\
&= \sum_{j_1=1}^{n_1} \cdots \sum_{j_{k+1}=1}^{n_{k+1}} \tensor{F}_{j_1,\ldots,j_{k+1}} ( \rho_1^{*}(g) \vb{e}_{j_1}^{1*} \vb{x}_1)  \otimes\cdots\otimes (\rho_k^{*}(g) \vb{e}_{j_{k}}^{k*} \vb{x}_k) \otimes (\rho_{k+1}(g) \vb{e}_{j_{k+1}}^{k+1})\\
&= \left( (\rho_1^*(g), \ldots, \rho_k^*(g), \rho_{k+1}(g)) \cdot \tensor{F} \right)_{(k+1)}(\vb{x}_1 \otimes \cdots \otimes \vb{x}_k).
\end{align*}
As this equality of linear maps holds for arbitrary $\vb{x}_i$, and $\tensor{F} \simeq \tensor{F}_{(k+1)}$, the proof is concluded.
\end{proof}

\begin{remark}
The concept of $G$-invariance in representation theory thus corresponds to equivariance in the machine learning community. The machine learning invariance concept corresponds to $G$-invariance where $M^{k+1}$ is the identity map $\Id$ (the trivial representation) for all $(M^{1},\ldots,M^{k+1})\in G$. 
\end{remark}

While we will only rely on \cref{prop_invariance} for studying the group-invariance of TTNs, the next definition is the natural generalization to arbitrary $(k,d)$-tensors.

\begin{definition}[$G$-invariant tensor] \label{def:def_invariance}
		Let $G$ be a finite group with representations $\rho_i : G \mapsto \mathrm{Aut}(\VS{W}_i)$. A tensor $\tensor{T} \in \VS{T}_{k,d}$ is called \emph{$G$-invariant} if
		\begin{align}\label{eq:eqn_g_invariant}
		\tensor{T} = (\rho^*_{1}(g), \dots, \rho^*_{k}(g), \rho_{k+1}(g), \dots, \rho_{d}(g)) \cdot \tensor{T}, \quad\forall g \in G
		\end{align}
or, equivalently,
		\[
		(\rho_{1}(g)^\top, \dots, \rho_{k}(g)^\top,\Id,\ldots,\Id) \cdot \tensor{T} 
		=
		( \Id, \dots, \Id, \rho_{k+1}(g), \dots, \rho_d(g)) \cdot \tensor{T}, \quad\forall g\in G.
		\]
\end{definition}

The two characterizations in \cref{def:def_invariance} are equivalent, which can be seen by multilinearly multiplying both sides of \cref{eq:eqn_g_invariant} with $(\rho_{1}(g)^\top, \dots, \rho_{k}(g)^\top, \Id,\ldots,\Id)$ and recalling that $\rho^*_i(g) = \rho^{-\top}_i(g)$.

\section{Efficiently constructing group-invariant TTNs}\label{sec:tnequiv}
Here, we first recall the definition of TTNs, and then introduce a new efficient method to find a basis of group-invariant tensors. 

\subsection{TTNs}\label{sec_tt_explain}
Let $k+1 = d$ for brevity.
By introducing $k-1$ auxiliary vector spaces $\VS{B}_i$, the TTN we consider represents a tensor $\tensor{F}$ in $\VS{T}_{k,d} = \VS{W}_1^* \otimes\dots\otimes \VS{W}_{k}^* \otimes \VS{W}_{d}$ using a sequence of low-order tensors (see also \cref{fig:tt_ml}) as follows:
\[
(\matr{T}^{1}, \tensor{T}^{2},\dots,\tensor{T}^{k-1},\matr{T}^{k}).
\]
Herein, at the left and right end, we have matrices $\matr{T}^{1} \in \VS{W}_1^* \otimes \VS{B}_1$ and $\matr{T}^{k} \in \VS{B}_{k-1} \otimes \VS{W}_k^*$, respectively. The \textit{output core tensor} $\tensor{T}^{\ell} \in \VS{W}_\ell^* \otimes \VS{B}_{\ell-1}^* \otimes \VS{B}_{\ell}^* \otimes \VS{W}_{k+1}$ is assumed to appear at position $1 < \ell < k$ in the TTN. This is usually taken to be the core in the middle as in \cref{fig:tt_ml}. The other core tensors are $\tensor{T}^{i} \in \VS{W}_i^* \otimes \VS{B}_{i-1}^* \otimes \VS{B}_i$ on the left for $1 < i < \ell$, and $\tensor{T}^{i} \in \VS{W}_i^* \otimes \VS{B}_{i-1}\otimes\VS{B}_i^*$ on the right for $\ell< i < k$. 

The tensor $\tensor{F}$ represented by the above TTN can be understood as explained next. 
We can express 
\(
 \tensor{F} = \sum_{j_{d}=1}^{n_{d}} \tensor{F}^{j_{d}} \otimes \vb{e}_{j_{d}}^{d},
\)
where $(\vb{e}_1^{d}, \dots,\vb{e}_{n_{d}}^{d})$ is a basis of the $n_{d}$-dimensional vector space $\VS{W}_{d}$; hence, $\tensor{F}^{j_{d}}$ is the $j_{d}$th ``slice'' of $\tensor{F}$ in the output factor $\VS{W}_{d}$. 
The tensors $\tensor{F}^{j_{d}}$ live in $\VS{W}_1^*\otimes\dots\otimes\VS{W}_k^* \simeq L(\VS{W}_1\otimes\dots\otimes\VS{W}_k; \RR)$, and are, hence, defined through their action on rank-$1$ tensors.
Let $(\vb{e}_{1}^{i*},\ldots,\vb{e}_{n_i}^{i*})$ be a basis of $\VS{W}_i^*$ as before.
The action of $\tensor{F}^{j_{d}}$ on $\vb{x}_1\otimes\dots\otimes\vb{x}_k$ is defined to be:
\begin{equation}\label{eqn_def_tt}
\tensor{F}^{j_{d}}(\vb{x}_1\otimes\dots\otimes\vb{x}_k) 
= 
\tensor{T}_{\vb{x}_{\ell}}^{\ell,j_{d}}\, 
\Bigl( 
\underbrace{\tensor{T}_{\vb{x}_{\ell-1}}^{\ell-1} \cdots \tensor{T}_{\vb{x}_2}^{2} \matr{T}^{1} \vb{x}_1}_{\text{a vector in }\VS{B}_{\ell-1}},\;
\underbrace{\tensor{T}_{\vb{x}_{\ell+1}}^{\ell+1} \cdots \tensor{T}_{\vb{x}_{k-1}}^{k-1} \matr{T}^{k} \vb{x}_k}_{\text{a vector in }\VS{B}_{\ell}} \Bigr),
\end{equation}
where, for $1 \le p \ne \ell \le k$,
\begin{align*}
\tensor{T}^{p}_{\vb{x}_p} &:= \sum_{j_p=1}^{n_p} \vb{e}_{j_p}^{p*}(\vb{x}_p) \cdot \matr{T}_{j_p}^{p} 
&&\quad\text{if } \tensor{T}^{p} = \sum_{j_p=1}^{n_p} \vb{e}_{j_p}^{p*} \otimes \matr{T}_{j_p}^{p}, \text{ and}\\
\tensor{T}^{\ell,j_{d}}_{\vb{x}_\ell} &:= \sum_{j_\ell=1}^{n_\ell} \vb{e}_{j_\ell}^{\ell *}(\vb{x}_\ell) \cdot \matr{T}_{j_\ell,j_{d}}^{\ell} 
&&\quad\text{if } \tensor{T}^{\ell} = \sum_{j_\ell=1}^{n_\ell}\sum_{j_{d}=1}^{n_{d}} \vb{e}_{j_\ell}^{\ell *} \otimes \matr{T}_{j_\ell,j_{d}}^{(\ell)} \otimes \vb{e}_{j_{d}}^{d}.
\end{align*}
Note that $\tensor{T}_{\vb{x}_\ell}^{\ell,j_{d}} \in \VS{B}_{\ell-1}\otimes\VS{B}_\ell$, so it can also be interpreted as a matrix in $\VS{B}_{\ell-1}^* \otimes \VS{B}_\ell$ satisfying $\tensor{T}_{\vb{x}_\ell}^{(\ell),j_{d}}(\vb{v}, \vb{w}) = \vb{w}^* \tensor{A}_{\vb{x}_\ell}^{\ell,j_{d+1}} \vb{v},$ where $\vb{w}^*$ is the dual vector of $\vb{w}$.
In coordinates, $\tensor{T}_{\vb{x}_p}^{p}$ is obtained by taking a linear combination with the coordinates of $\vb{x}_\ell$ of the appropriate ``slices'' of the tensor $\tensor{T}^{p}$. For more details see Oseledets \cite{oseledets}.

The dimensions of the auxiliary vector spaces $\VS{B}_i$ are called the \textit{bond dimensions}. Hence, $\tensor{F}$ above is said to have bond dimensions $(r_1,\dots,r_{d-1})$ where $r_i = \dim \VS{B}_i$. Every tensor can be represented exactly in TTN format with sufficiently large bond dimensions \cite{Verstraete_2004}. Note that the minimal bond dimension with which a tensor $\tensor{F}$ can be represented is called the TTN rank \cite{YL2018}. We do not assume minimality of the bond dimensions, hence we consistently refer to $r_i$ as a bond dimension rather than a rank.

\subsection{$G$-invariant TTNs} \label{sec_equivariant_tt} 
We introduce an efficient scheme to construct TTNs that represent tensors invariant under a \textit{finite} discrete group $G$ with normal representation matrices. The construction of such networks is simplified by the next result.

\begin{lemma}[Singh, Pfeifer, and Vidal \cite{singh2010tensor}] \label{lem_g_equiv_tt}
Let $G$ be a finite discrete group. The tensor represented by a tensor network with $G$-invariant cores is $G$-invariant. 
\end{lemma}

Moreover, a $G$-invariant tensor decomposed as a tree tensor network (which includes TTNs) with $G$-invariant cores has the minimal bond dimension achievable by any tree tensor network decomposition with the same tree architecture \cite{sukhwinder2013}. 
The problem of finding a tree tensor network representing a $G$-invariant tensor is thus reduced to finding a tree tensor network with $G$-invariant cores. We call such tensor networks \textit{$G$-invariant tensor networks}. 

We describe an approach for the case of $G$-invariant tensors from \cref{def:def_invariance}, even though for TTNs only \cref{prop_invariance} would suffice.
It follows from \cref{def:def_invariance,eqn_multilinear_multiplication} that a $G$-invariant tensor satisfies the \textit{joint eigenproblem}
\begin{align}\label{eqn_joint_eigenproblem}
\left( \rho_{1}^{*}(g) \otimes\dots\otimes \rho_k^{*}(g) \otimes \rho_{k+1}(g) \otimes\dots\otimes \rho_d(g) \right)(\tensor{T}) = \tensor{T}
\end{align}
for all group elements $g \in G$ simultaneously. 
If the finite group $G$ is generated by the elements $g_1, \dots, g_s$, i.e., $G=\langle g_1, \dots, g_s \rangle$, then it suffices that \cref{eqn_joint_eigenproblem} holds for all generators $g_i$ to imply that it holds for all $g\in G$; see \cref{sec_sub_groupstructure}.
This entails that $G$-invariant tensors live in a linear subspace (called a \textit{subrepresentation} \cite{FH2004}). 
The approach of Finzi, Welling, and Wilson \cite{finzi2021practical} consists of computing an orthonormal basis of this subrepresentation and is briefly described in \cref{sec:rel_work} below. Here, we describe an alternative, more efficient approach for constructing such a basis.

For normal representations, we can proceed as follows.
Consider the simplest case of a cyclic group $G=\langle g_1 \rangle$ for which we want to find a basis of the solution space $\X$ of the following eigenproblem:
\begin{equation}\label{eq:probcyclic}
(\U^{1}_1\otimes \dots  \otimes\U^{d}_1)\X = \X,
\end{equation}
where the matrices $\U^{i}_1 \in \RR^{n_i \times n_i}$ are normal representation matrices for the generator $g_1$ on $\VS{W}_i$ or $\VS{W}_i^*$, i.e., $\U^{i}_1 = \rho_i^*(g_1) = \rho_i^{-\top}(g_1)$ if $i=1,\ldots,k$ and $\U^{i}_1 = \rho_i(g_1)$ if $i=k+1,\ldots,d$. Since the representation matrices are normal, they are \textit{unitarily} diagonalizable as $\U_1^{i} = \V^{i}\L^{i}(\V^{i})^H$ where $\L^{i}$ is diagonal (with all eigenvalues on the complex unit circle) and $\V^{i}$ a unitary matrix (so $(\V^{i})^H \V^{i} = \Id = \V^{i} (\V^{i})^H$); recall that superscripts are indices. Consequently, because of the tensor product, we have
\begin{equation}\label{eq:cond_cycl}
	(\V^{1} \otimes \dots \otimes \V^{d})
	(\L^{1} \otimes \dots  \otimes \L^{d})
	(\V^{1} \otimes \dots \otimes \V^{d})^H\X = \X.
\end{equation}
The solutions of \cref{eq:probcyclic} are the tensor products of eigenvectors for which the product of corresponding eigenvalues equals $1$.
Let $\{ (j_{q}^{1},\dots,j_{q}^{d}) \}_{q=1}^p$ be all indices such that 
\begin{equation}\label{eqn_prodone}
\Lambda^{1}_{j_{q}^{1},j_{q}^{1}}\cdots\Lambda^{d}_{j_{q}^{d},j_{q}^{d}} = 1, \text{ and let }  \matr{V}_*^{i} := \begin{bmatrix} \vb{v}_{j_{1}^{i}}^{i} & \cdots & \vb{v}_{j_{p}^{i}}^{i} \end{bmatrix} \in \RR^{n_i \times p},
\end{equation}
where $\vb{v}_{j}^{i}$ is the $j$th column of $\matr{V}^{i}$.
Then, the solutions of \cref{eq:probcyclic} can be put as columns into the following $n_{1}\cdots n_d \times p$ matrix:
\[
\matr{V}^{1}_* \odot \dots\odot \matr{V}^{d}_* :=
 \begin{bmatrix} \vb{v}_{j_{1}^{1}}^{1}\otimes \dots \otimes \vb{v}_{j_{1}^{d}}^{d} & \cdots & \vb{v}_{j_{p}^{1}}^{1}\otimes \dots \otimes \vb{v}_{j_{p}^{d}}^{d} \end{bmatrix},
\]
where $\odot$ is called the (columnwise) Khatri--Rao product. Note that $\matr{V}^{1}_* \odot\dots\odot \matr{V}^{d}_*$ contains a subset of the columns of $\matr{V}^{1}\otimes\dots\otimes\matr{V}^{d}$, from which it follows that the columns of $\matr{V}^{1}_* \odot\dots\odot \matr{V}^{d}_*$ form an orthonormal (in the Hermitian inner product) basis of the solution space of \cref{eq:probcyclic}.

For a group $G = \langle g_1, \dots, g_s \rangle$ with multiple generators the problem we have to solve is a joint eigenproblem
\begin{align}\label{eqn_orig_problem}
(\U^{1}_{1} \otimes \dots \otimes \U^{d}_{1})\X= \X,\quad \ldots,\quad
(\U^{1}_{s} \otimes \dots \otimes \U^{d}_{s})\X=\X.
\end{align}
where the matrix $\U^{i}_j \in \RR^{n_i \times n_i}$ is a normal representation matrix for the generator $g_j$ on $\VS{W}_j$, as before.
From the cyclic case, we know that $\X$ must live in the space spanned by the columns of $\matr{V}_*^{1}\odot\dots\odot\matr{V}_*^{d}$. That is, we can \emph{restrict to this subspace} and impose that $\X = (\V^{1}_* \odot \dots \odot \V^{d}_*)\Q$. Plugging this into the other equations yields
\begin{equation}\label{eq:partstelsel}
\bigl( (\U^{1}_{j}\V^{1}_*) \odot \dots \odot (\U^{d}_{j}\V^{d}_*) \bigr)\Q\ 
= (\V^{1}_* \odot \dots \odot \V^{d}_*)\Q, \quad j=1,\ldots,s.
\end{equation}
For each $j$ individually, this is a generalized rectangular eigenproblem \cite{10.1093/imamat/20.4.443} that can be solved as a regular eigenvalue problem by multiplying the left-hand side with the pseudo-inverse of the matrix on the right. Since the right matrix has orthogonal columns (with respect to the Hermitian inner product), its pseudo-inverse is the conjugate transpose. Hence, the system \cref{eq:partstelsel} is equivalent to the $p \times p$ eigenproblem
\begin{equation}\label{eq:systemeq}
\bigl( ( (\V_*^{1})^H\U_{j}^{1}\V^{1}_*) \circledast \cdots \circledast ((\V_*^{d})^H\U_{j}^{d}\V^{d}_*) \bigr)\Q = \Q, \quad j=1,\ldots,s;
\end{equation}
herein, we exploited the basic property of the Khatri--Rao product that $(A \odot B)^H (C \odot D) = (A^H C) \circledast (B^H D)$, where $\circledast$ is the Hadamard or elementwise product. This first reduction step from \cref{eqn_orig_problem} to \cref{eq:systemeq} is especially interesting computationally when $p$ is much smaller than $N=n_1 \cdots n_d$, as we will see below.

The matrices $((\V_*^{1})^H\U^{1}_{j}\V^{1}_* \circledast \dots \circledast (\V_*^{d})^H\U^{d}_{j}\V^{d}_*)$ are usually not normal and therefore not unitarily diagonalizable. Repeating the subspace restriction step is therefore, unfortunately, not possible.
However, solving the non-normal eigensystem \cref{eq:systemeq} is vastly simplified by the next result.\footnote{Representation theorists may note the connection to the averaging endomorphism \cite[Section 2.2]{FH2004} for representations, where the rightmost real eigenvalue can only be $\lambda_1=\dots=\lambda_s=1$.}

\begin{proposition}\label{prop_reduction}
Let $\B_i \in \CC^{m \times m}$, $i=1,\ldots,s$, be normal matrices whose rightmost eigenvalues are real and positive:
\[
 (\lambda_i, \vb{w}_i) := \arg \max_{\substack{(\lambda,\vb{w})\in\CC\times\mathbb{S}^n\\ B_i \vb{w} = \lambda \vb{w}}} \mathfrak{R}\left( \vb{w}^H \B_i \vb{w} \right) \in \RR_+\times\mathbb{S}^n;
\]
herein, $\mathbb{S}^{n} = \{ \vb{w}\in\CC^n \mid \Vert\vb{w}\Vert=1 \}$ and $\mathfrak{R}(a)$ denotes the real part of $a\in\CC$.
Let $\U \in \CC^{m \times n}$ be a matrix with orthonormal columns in the Hermitian inner product, and let $\A_i = \U^H \B_i \U \in \CC^{n \times n}$. 

The vector $\vb{v} \in \mathbb{S}^n$ solves the joint eigenvector problem
\[
 \A_1 \vb{v} = \lambda_1 \vb{v},\quad
 \A_2 \vb{v} = \lambda_2 \vb{v}, \quad 
 \ldots, \quad 
 \A_s \vb{v} = \lambda_s \vb{v}
\] if and only if $(x_1 \A_1 + \dots + x_s \A_s) \vb{v} = \vb{v}$ with $x_i = \frac{1}{s \lambda_i}$.
\end{proposition}
\begin{proof}
If $\vb{v}$ is a solution of the joint eigenproblem, then taking a linear combination of the system with the coefficients $x_i=\frac{1}{s\lambda_i}$ yields
\begin{align*}
x_1 \A_1 \vb{v} + \cdots + x_s \A_s \vb{v} 
= x_1 \lambda_1 \vb{v} + \cdots + x_s \lambda_s \vb{v} 
=\vb{v}.
\end{align*}

Conversely, because $x_i\in\RR$, we have
\begin{align*}
1 = \vb{v}^H \left( x_1 \A_1 + \cdots + x_s \A_s \right)\vb{v}
= \sum_{i=1}^s x_i \mathfrak{R}\left( \vb{v}^H \A_i \vb{v} \right)
= \sum_{i=1}^s x_i \mathfrak{R}\left( (\U \vb{v})^H \B_i (\U \vb{v}) \right).
\end{align*}
Since $\B_i$ is a normal matrix, its \textit{numerical range} or \textit{field of values} is the convex hull of its eigenvalues \cite[Chapter 1]{HJ1991}. Consequently, 
\begin{align*}
 \lambda_i 
 = \max_{\substack{(\lambda,\vb{w})\in\CC\times\mathbb{S}^n\\ \B_i \vb{w} 
 = \lambda \vb{w}}} \mathfrak{R}\left( \vb{w}^H \B_i \vb{w} \right) 
 = \max_{\vb{w}\in\mathbb{S}^n} \mathfrak{R}\left( \vb{w}^H \B_i \vb{w} \right).
\end{align*}
Combining the two foregoing equations, we obtain
\[
 1 
 = \sum_{i=1}^s x_i \mathfrak{R}\left( (\U \vb{v})^H \B_i (\U \vb{v}) \right)
 \le \sum_{i=1}^s x_i \max_{\vb{w}\in\mathbb{S}^n} \mathfrak{R}\left( \vb{w}^H \B_i \vb{w} \right) 
 = x_1 \lambda_1 + \cdots + x_s \lambda_s 
 = 1.
\]
Hence, the inequality is an equality, and subtracting both sides yields
\[
 \sum_{i=1}^s x_i \left( \lambda_i - \mathfrak{R}\left( (\U\vb{v})^H \B_i (\U\vb{v}) \right) \right) = 0.
\]
Since $x_i > 0$ this equation can hold only if $\lambda_i = \mathfrak{R}\left( (\U\vb{v})^H \B_i (\U\vb{v}) \right)$ for all $i=1,\ldots,s$. Since $\Vert \U \vb{v} \Vert = \Vert\vb{v}\Vert = 1$, we conclude that
\(
 (\lambda_i, \U \vb{v})
\)
is a solution of the eigenproblem associated to $\B_i$, for all $i$. That is,
\[
 \B_i (\U \vb{v}) = \lambda_i (\U \vb{v}), \quad i=1,\ldots,s.
\]
Multiplying both sides of the equality by $\U^H$ concludes the proof.
\end{proof}

We can apply this result to the normal matrices $\B_j = \U_j^{1}\otimes\cdots\otimes \U_j^{d}$ and matrix $\U = \V_*^{1}\odot\cdots\odot \V_*^{d}$ with orthonormal columns.
It follows that the solutions of the system \cref{eq:systemeq} are the solutions of the \textit{single}, averaged standard eigenproblem
\[
 \frac{1}{s} \sum_{j=1}^s \bigl( ( (\V_*^{1})^H\U_{j}^{1}\V^{1}_*) \circledast \cdots \circledast ( (\V_*^{d})^H \U_{j}^{d}\V^{d}_*) \bigr)\Q = \Q.
\]

\begin{algorithm}[t]
	\caption{Construct a basis of $G$-invariant tensors}
	\label{alg:main}
	\begin{algorithmic}[1]
		\REQUIRE{Normal representation matrices $\U^{1}_j, \ldots, \U^{d}_j$ of the generator $g_j$, for $j=1,\ldots,s$.}
		\STATE{For $i=1,\ldots,d$, compute the eigendecomposition $\U^{i}_{1} = \V^{i}\L^{i}(\V^{i})^H$.}
		
		\STATE{For $i=1,\ldots,d$, compute $\V_*^{i}$ as in \cref{eqn_prodone}.}
		
		\STATE{Compute $\A = \frac{1}{s} \sum_{j=1}^s ( (\V_*^{1})^H\U^{1}_{j}\V^{1}_*) \circledast \dots \circledast ( (\V_*^{d})^H \U^{d}_{j}\V^{d}_*)$.}
		
		\STATE{Compute a Schur decomposition $\A = \V \T \V^{H}$ and extract an orthonormal basis $\Q \in \CC^{p \times r}$ (so $\Q^H Q = \Id$) for the eigenspace corresponding to eigenvalue $1$.}
		
		\ENSURE{The orthonormal basis $(\V^{1}_* \odot \dots \odot  \V^{d}_*)\Q$}
	\end{algorithmic}
\end{algorithm}

We summarize the above discussion as an algorithm for finding a basis of the $G$-invariant tensors as \cref{alg:main}.
It leaves some space for heuristic optimizations. For example, the choice of which generator $g_j$ should be processed first (we assumed $g_1$ above) has a serious impact on the computational cost. Indeed, the number of indices $p$ in step 2 of \cref{alg:main} determines the size of the reduced eigenproblem that is solved in step $4$; it is $p \times p$. Hence, it is interesting to choose the first generator as the one that minimizes $p$.
Another observation is that the representations are often the same when some or all of the vector spaces $\VS{W}_i$ coincide, as in our numerical experiments in \cref{sec:experiments}. This ensures that some calculations can be recycled. We did not implement these optimizations. 

There are also numerical issues relevant to \cref{alg:main}: how to decide when an eigenvalue is numerically equal to one or not. In our implementation, any eigenvalue $\lambda$ for which $\vert \lambda - 1 \vert < 10^{-6}$ is treated as an eigenvalue $1$ for the purpose of \cref{alg:main}.

The resulting algorithm has an asymptotic time complexity of order
\begin{displaymath}
\underbrace{(n_{1}^3 + \dots + n_d^3)}_{\text{step 1}} + 
\underbrace{N}_{\text{step 2}} + 
\underbrace{2(n_{1}^2 + \dots + n_d^2)ps + d p^2 s}_{\text{step 3}} + 
\underbrace{p^3}_{\text{step 4}} .
\end{displaymath}
Herein, $s$ is the number of generators, $n_i$ is the dimension of the $i$th vector space, $N = n_{1} \cdots n_d = \dim\VS{T}_{k,d}$ is the tensor dimension, $d$ is the tensor order, and $p$ is the number of combinations for which the eigenvalue is $1$.

For a cyclic group, the complexity reduces to $O((n_{1}^3 + \dots + n_d^3) + N)$, this is a strict improvement over the method from \cite{finzi2021practical}. For groups with multiple generators a theoretical comparison is more difficult. In theory, if every eigenvalue of all generators is $1$, then $p = N$, so the asymptotic time complexity of \cref{alg:main} is worse than the state-of-the-art method from \cite{finzi2021practical}. However, in practice, $p$ will never be as large as $N$, and our numerical experiments in \cref{sec:experiments} indicate orders of magnitude improvements over \cite{finzi2021practical} in practical settings. We discuss the value of $p$ in more depth in \cref{sec:pval}.

\section{Orthogonal representations and a good first generator}\label{sec:pval}
The crux of \cref{alg:main} is that oftentimes the amount of initial combinations of eigenvalues equal to $1$, i.e., $p$ in \cref{alg:main}, can be kept relatively small. In this section, we investigate $p$ and introduce some heuristics to keep it small.

The key observation is that \cref{alg:main} applies for any set of generators of $G$. In particular, there is no assumption of minimality. Hence, if $G = \langle g_2, \ldots, g_{s+1} \rangle$, then we are free to add any additional group element $g_1$, so that $G = \langle g_1, \ldots, g_{s+1} \rangle.$ We can thus assume without loss of generality that $g_1 \in G$ in \cref{sec_equivariant_tt} is any desirable group element, albeit at the cost of a minor impact on the computational performance as the number of generators increases by $1$. The trick is then to choose $g_1$ in a way that minimizes $p$.

Representation theory studies representations of groups in different dimensions. 
Here, we focus on a few basic, well-known groups of order $n$ that have natural $n$-dimensional orthogonal representations:\footnote{It is important to note that groups of order $n$ can also have representations in different dimensions than $n$. For example, the cyclic group $C_n = \{g_1, \ldots, g_n\}$ has a natural $1$-dimensional representation for all $n \in \NN$: $\U_{j} = e^{2\pi \imath\frac{j}{n}}$.} the \textit{cyclic group} $C_n$, the \textit{dihedral group} of symmetries of a regular $n$-gon $D_n$, the \textit{symmetric group} of permutations on $n$ symbols $S_n$, and the \textit{dicyclic group} $Q_n$ (which includes as a special case the group of \textit{quaternions}). The dicyclic group exists only for $n$ divisible by $4$. We recall these well-known representations next.

Consider the permutation matrices
\begin{align}\label{eqn_Uc_and_Ur}
\U_{c} = \begin{pmatrix}
0 & 1 & 0 & \cdots & 0 \\
0 & 0 & 1 &\ddots & \vdots \\
\vdots& \vdots & \ddots &   \ddots &0\\
0 & 0 & \cdots& 0 & 1\\
1 & 0 &  \cdots & 0 & 0
\end{pmatrix} \quad\text{ and }\quad
\U_{r} = \begin{pmatrix}
0 & \dots & 0 & 1 \\
\vdots& \iddots &   \iddots & 0\\
0 &   \iddots & \iddots &\vdots\\
1 & 0 &  \dots  & 0
\end{pmatrix}.
\end{align}
Applied to a vector $\vb{v}$, $U_c$ circularly shifts the elements of $\vb{v}$ upwards, and $U_r$ inverts the order of the elements of $\vb{v}$. Let $\U_{s}^{ij} = 
\Id - (\vb{e}_i - \vb{e}_j)(\vb{e}_i - \vb{e}_j)^\top$ be the orthogonal matrix that, when applied to $\vb{v}$ swaps the elements at positions $i$ and $j$ of $\vb{v}$. Then, 
\begin{itemize}
 \item[$\blacktriangleright$] $C_n = \langle g_c \rangle$ is generated by a single generator $g_c$ whose $n$-dimensional representation is the cyclic shift matrix $\rho(g_c)= \U_c$.
 \item[$\blacktriangleright$] $D_n = \langle g_c, g_r \rangle$ is generated by $g_c$ and the additional generator $g_r$ whose representation is the reverser matrix $\rho(g_r) = \U_r$. 
 \item[$\blacktriangleright$] $S_n = \langle g_c, g_s^{12} \rangle$ is generated by $g_c$ and the additional generator $g_s^{12}$ whose representation is the swapping matrix $\rho(g_s^{12}) = \U_s^{12}$.
 \item[$\blacktriangleright$] $Q_n = \langle g_c, g_x \rangle$ is generated by $g_c$ and an additional generator $g_x$ that satisfies the relations $g_x^2 = g_c^\frac{n}{2}$ and $g_c g_x = g_x g_c^{-1}$.
\end{itemize}

All four of these groups have $g_c$, a cyclic shift, as generator. We observe that $\U_c$, the representation of $g_c$ in dimension $n$, is a \textit{circulant matrix}. The eigenvalues of such matrices are known; see, for example, \cite{davis1979circulant}.

\begin{lemma}[Davis \cite{davis1979circulant}]
 The eigenvalues of $\U_c$ are $\lambda_j = e^{-2\pi \imath \frac{1}{n} j}$.
\end{lemma}

\begin{corollary}
Assume that we choose $g_c$ as first generator of $C_n$, $D_n$, $S_n$, or $Q_n$. When $g_c$ is represented in dimension $n$ by $\U_c$, then \cref{eq:cond_cycl} has a solution space of dimension $p = n^{d-1}$.
\end{corollary}
\begin{proof}
 Let $1 \le i_1, i_2, \ldots, i_{d-1} \le n$ be arbitrary. Then,
 \[
  \lambda_{i_1} \lambda_{i_2} \cdots \lambda_{i_{d-1}} = e^{-2\pi\imath\frac{1}{n}(i_1+i_2+\cdots+i_{d-1})} = e^{-2\pi\imath\frac{1}{n}(i_1 + i_2 + \cdots + i_{d-1}\mod n)},
 \]
where the second equality is due to the fact that the group of $n$th roots of unity is cyclic of order $n$. It follows there is precisely one $1 \le i_d \le n$ such that $i_1 + i_2 + \dots + i_{d-1} + i_d = 0 \mod n$. For this choice of $i_d$, $\lambda_{i_1}\cdots\lambda_{i_d}=1$, while any other choice leads to a different $n$th root of unity. 
Hence, there are exactly $p=n^{d-1}$ products of eigenvalues equal to $1$ in $\Lambda^{1}\otimes\cdots\otimes\Lambda^{d}$ in \cref{eq:cond_cycl}, concluding the proof.
\end{proof}

Since the maximum value of $p$ is $\dim\VS{T}_{k,d} = n^{d} = N$, the previous result causes a significant reduction of a factor of $\frac{1}{n}$ in the size of the system \cref{eq:partstelsel} relative to \cref{eqn_orig_problem}. Choosing a different generator than $g_c$ can result in a much poorer reduction. For example, because $\U_s^{12}=\mathrm{diag}\bigl(\left(\begin{smallmatrix}0&1\\1&0\end{smallmatrix}\right), \Id \bigr)$, we see that it has $n-2$ eigenvalues equal to $1$. In this case, \cref{eq:cond_cycl} will have at least $(n-2)^d \approx n^d = N$ solutions, causing no appreciable reduction in size.

A natural way to construct new discrete groups is from the product of smaller discrete groups.  
For the product group $G \times H$ of $G = \langle g_1, \dots, g_{s_1}\rangle$ and $H = \langle h_1, \dots, h_{s_2}\rangle$, two representations can easily be found from the representations of $G$ and $H$.
Let $\rho_G$ be a $d_1$-dimensional representation of $G$ and $\rho_H$ a $d_2$-dimensional representation of $H$. Then, we can either take the \textit{direct sum}, $\rho_{G\times H}(g,h) = \rho_G(g) \oplus \rho_H(h) = \mathrm{diag}(\rho_G(g), \rho_H(h))$ as a $(d_1+d_2)$-dimensional representation, or take the \textit{tensor product}, $\rho_{G\times H}(g,h) = \rho_G(g) \otimes \rho_H(h)$, as a $d_1d_2$-dimensional representation.

For the product group, $G \times H = \langle (g_1, e_h), \dots, (g_{s_1}, e_h), (e_g, h_1), \dots, (e_g, h_{s_2})\rangle$ is a natural generating set, where $e_g$ is the neutral element of $G$ and likewise for $e_h\in H$. 
However, for our purpose, it can be useful to add the element $(g_1, h_1)$ to the generating set and use this as the first generator. This is because, for the direct sum case, in $\rho_G(g_1)\oplus\rho_H(e_H)$, the $\rho_H(e_H)$ part is the identity matrix, hence adding the maximum number of eigenvalues equal to $1$ in the direct sum. By contrast, the number of eigenvalues $1$ in $\rho_G(g_1)\oplus\rho_H(h_1)$ is the sum of those in $\rho_G(g_1)$ and $\rho_H(h_1)$. In the tensor product case, $(g_1, h_1)$, whose representation is $\rho_G(g_1) \otimes \rho_H(h_1)$, is an evident choice as first generator. The number of eigenvalues equal to $1$ will be at least equal to the product of the number of eigenvalues $1$ of $\rho_G(g_1)$ and $\rho_H(h_1)$.

Finally, the different representations of groups and their corresponding differences in the number of eigenvalues equal to $1$ is illustrated next.
\begin{example}
The octahedral group $O_h$ is the symmetry group of the octahedron and has three generators, commonly denoted by $a, b$ and $c$. A representation in $\RR^3$ is
{\small\begin{align*}
	\rho(a) = 
	\begin{pmatrix}
	-1 & 0 & 0\\
	0 & -1 & 0 \\
	0 & 0 & -1 \\
	\end{pmatrix},\;
	\rho(b) =
	\begin{pmatrix}
	0 & 0 & 1\\
	1 & 0 & 0 \\
	0 & 1 & 0 \\
	\end{pmatrix},\;
	\rho(c) =	
	\begin{pmatrix}
	-1 & 0 & 0\\
	0 & -1 & 0 \\
	0 & 0 & 1 \\
	\end{pmatrix}.
\end{align*}}
Other representations can be found from the fact that $O_h \simeq S_4 \times C_2$. For example, a direct sum representation in 5 dimensions is
{\small\begin{align*}
\rho_\oplus(a) &= 
\begin{pmatrix}
0 & 1 & 0 & 0 & 0\\
1 & 0 & 0 & 0 & 0\\
0 & 0 & 1 & 0 & 0\\
0 & 0 & 0 & 1 & 0 \\
0 & 0 & 0 & 0 & 1 \\
\end{pmatrix},\;
\rho_\oplus(b) =
\begin{pmatrix}
1 & 0 & 0 & 0 & 0\\
0 & 0 & 0 & 1 & 0\\
0 & 1 & 0 & 0 & 0\\
0 & 0 & 1 & 0 & 0 \\
0 & 0 & 0 & 0 & 1 \\
\end{pmatrix},\;
\rho_\oplus(c) =	
\begin{pmatrix}
1 & 0 & 0 & 0 & 0\\
0 & 1 & 0 & 0 & 0\\
0 & 0 & 1 & 0 & 0\\
0 & 0 & 0 & 1 & 0 \\
0 & 0 & 0 & 0 & -1 \\
\end{pmatrix},
\end{align*}}%
and a tensor product representation in 8 dimensions is
{\small\begin{align*}
\rho_\otimes(a) = 
\begin{pmatrix}
0 & 1 & 0 & 0 \\
1 & 0 & 0 & 0 \\
0 & 0 & 1 & 0\\
0 & 0 & 0 & 1 \\
\end{pmatrix} \otimes \Id_2,\;
\rho_\otimes(b) = 
\begin{pmatrix}
1 & 0 & 0 & 0 \\
0 & 0 & 0 & 1 \\
0 & 1 & 0 & 0 \\
0 & 0 & 1 & 0 \\
\end{pmatrix} \otimes \Id_2,\;
\rho_\otimes(c) =
\Id_4 \otimes 
\begin{pmatrix}
0 & 1 \\
1 & 0 \\
\end{pmatrix}.
\end{align*}}

\Cref{tab:ps} lists the number of products of eigenvalues equal to $1$, i.e., $p$ in \cref{alg:main}, with the different representations and different choices of first generator. As can be seen, the choice of representation and choice of first generator have a major impact on the value of $p$.

\begin{table}[t]\footnotesize
\caption{\footnotesize The number of eigenvalues equal to $1$ in $\U^{\otimes d} = \U \otimes \cdots \otimes \U$ for increasing $d$ and different representations $\U$ of the octahedral group in dimensions $d$. This is the value of $p$ in \cref{alg:main}.} \label{tab:ps}
\begin{center}
\begin{tabular}{cccccccccc}
\toprule
$d$ & \multicolumn{9}{c}{$\U$} \\
\cmidrule{2-10}
 & $\rho(a)$ & $\rho(b)$ & $\rho(c)$ & $\rho_\oplus(a)$ & $\rho_\oplus(b)$ & $\rho_\oplus(c)$ & $\rho_\otimes(a)$ & $\rho_\otimes(b)$ & $\rho_\otimes(c)$ \\  
$2$ & 9 & 3 & 5 &   17 & 5 & 17 & 40 & 24  & 32\\
$3$ & 0 & 9 & 13 & 76 & 47 & 76 & 288 & 176 & 256 \\
$4$ & 81 & 27 & 41 & 353 & 219 & 353 & 2176 & 1376 & 2048 \\
\bottomrule
\end{tabular}
\end{center}
\end{table}
\end{example} 

\section{Related work}\label{sec:rel_work} 
The most studied approach for imposing group-invariance on tensors in the machine learning literature consists of theoretically computing a decomposition of the representation of the group generated by $(M_g^1,\ldots,M_g^{k+1})$ from \cref{sec_sub_groupstructure} into \textit{irreducible subrepresentations} \cite[Chapter 1]{FH2004}; see for example \cite{TSKYLKR2018,Kondor2018,FWFW2020,singh2010tensor}.  
The subrepresentation containing all group-invariant tensors is then the one corresponding to the \textit{trivial subrepresentation} \cite[Chapter 2]{FH2004}.
While there is a nice averaging characterization \cite[Section 2.2]{FH2004} of the trivial subrepresentation, in general, it is not trivial to find a basis for it. Similarly, theoretically decomposing an arbitrary (tensor product) representation into irreducible subrepresentations is considered a hard problem, called the Clebsch--Gordan problem; see \cite[Section 25.3]{FH2004} for more details and some available techniques that could be used.

A completely different, numerical approach was proposed by Finzi, Welling, and Wilson \cite{finzi2021practical}. 
They developed a numerical algorithm to construct invariant linear maps and tensors for any group. Our approach can be viewed as a more elaborate, yet more efficient version of their approach.
The approach in \cite{finzi2021practical} consists of rewriting \cref{eqn_joint_eigenproblem} as 
\begin{equation}\label{eq:const}
0 = 
\underbrace{\begin{pmatrix}
\U^{1}_1 \otimes \dots \otimes \U^{d}_1 - \Id \\
\vdots\\
\U^{1}_s \otimes \dots \otimes \U^{d}_s - \Id \\
\end{pmatrix}}_\matr{C} \tensor{T},
\quad\text{where }
\matr{U}_{j}^{i} = \begin{cases}
                  \rho_i^*(g_j) & \text{if } 1 \le i \le k, \\
                  \rho_i(g_j) & \text{if } k+1 \le i \le d, 
                 \end{cases}
\end{equation}
$\Id$ is the identity and $G = \langle g_1, \dots, g_s \rangle$. The valid tensors $\tensor{T}$ are thus the elements in the kernel of the \emph{constraint matrix} $\matr{C}$. This defines a standard problem in numerical linear algebra: Compute a basis of the kernel. After extracting a basis, every $G$-invariant tensor $\tensor{T}\in\VS{T}_{k,d}$ is a linear combination of these basis elements.

Since \cref{alg:main} is intended as an alternative to Finzi, Welling, and Wilson's method \cite{finzi2021practical}, it is instructive to compare their theoretic running times.
Let $N = n_{1} \cdots n_d$ be the dimension of $\VS{T}_{k,d}$. Then, $\matr{C}$ above is an $sN\times N$-matrix. Computing an orthonormal basis of the kernel of this matrix directly with the singular value decomposition would require $\mathcal{O}( s N^{3} )$ operations and $\mathcal{O}(sN^2)$ memory. The work \cite{finzi2021practical} attempts to circumvent this cost by employing a (nonstandard) Krylov method similar to gradient descent. In their approach only the action of $\matr{C}$ on a tensor is required, which essentially amounts to applying $s$ multilinear multiplications to $\tensor{T}$ and then subtracting $\tensor{T}$ from each of them. This avoids constructing the constraint matrix $\matr{C}$ explicitly. Some algorithms for multilinear multiplication or tensor-times-matrix (TTM) product are known \cite{tensorlab,JSSv087i10,LBPSV2015,BK2006}. The usual algorithm consists of $d$ successive matrix multiplications and requires $\mathcal{O}(N(n_{1}+\dots+n_d))$ operations, so that $\matr{C}$ could be applied with cost $\mathcal{O}(sN(n_{1}+\dots+n_d))$. Unless the spectrum of $\matr{C}$ is exceptionally well-clustered, one should expect that the number of iterations for any Krylov method to reliably extract a vector in the kernel is an appreciable fraction $\alpha > 0$ of the size of the matrix \cite{VVM2015}. Hence, the practical time complexity to find one vector in the kernel is expected to be $\mathcal{O}(\alpha s (n_{1}+\dots+n_d) N^2)$, and for $r$ vectors the cost is multiplied by $r$.  
Ultimately, this approach reduces the complexity from cubic in $N$ for a straightforward singular value decomposition to quadratic in $N=\dim\VS{T}_{k,d}$. 

In addition to the high computational cost, another practical drawback of the method from \cite{finzi2021practical} is that an initial guess about the dimension $r$ of the kernel of $\matr{C}$ needs to be made. If the guess is too small, the procedure is repeated with a doubled rank $r$ until the rank is big enough \cite{finzi2021practical}.  

\section{Experimental results}\label{sec:experiments}
In the next subsection, we compare the performance of \cref{alg:main} to the method introduced in \cite{finzi2021practical}, whose key steps were described in \cref{sec:rel_work}. In \cref{sec_sub_parity}, we compare the performance of group-invariant TTNs to their non-invariant, but data-augmented counterparts in a simple supervised learning task. Finally, in \cref{sec_suplearn_tt_expers}, we apply group-invariant TTNs as supervised learning model to a DNA transcription factor binding classification problem and compare them against a state-of-the-art group-invariant \textit{neural} network from \cite{mallet2021reversecomplement}.

We implemented all employed methods in Python $3.8$ using Tensorflow $2.6$ and SciPy and executed them on $10$ cores of one Xeon Gold $6140$ CPU ($18$ cores, clockspeed of each core $2.3$GHz, $24.75$MB L3 cache) with $180$GB RAM from the tier-2 Genius cluster of the Vlaams Supercomputer Centrum. Our implementation is available at \url{https://gitlab.kuleuven.be/u0134300/group-invariant-tensor-trains}.

\subsection{Group-invariant tensor basis construction}\label{sec:basisconst}
\Cref{alg:main} is compared to the state-of-the-art ``baseline'' method of Finzi, Welling, and Wilson \cite[Algorithm 1]{finzi2021practical} for group-invariant tensor basis construction for three groups over a range of orders and dimensions. 
For this baseline method, we used a compressed sparse row sparse matrix data structure for the representation matrices $\U_j^{i}$, while our method operates on dense matrices.
We also implemented a ``naive'' method that computes the nullspace of the constraint matrix $C$ in \cref{eq:const}---constructed explicitly with Kronecker products of dense matrices---using a dense singular value decomposition.

We compare the timings of \cref{alg:main} to the naive and baseline methods for cyclic groups $C_n = \langle g_c \rangle$, dihedral groups $D_n = \langle g_c, g_r \rangle$, and symmetric groups $S_n = \langle g_c, g_s^{12} \rangle$; the generators and representations we used were explained in \cref{sec:pval}. In these cases, $g_1 = g_c$ is selected as first generator in \cref{alg:main}. We additionally investigate the alternative set of generators $S_n = \langle g_s^{12}, g_s^{13}, \ldots, g_s^{1n} \rangle$ of the symmetric group $S_n$. Here we use $g_1=g_s^{12}$ as first generator.
 In all experiments, the dimensions and representations are the same for every vector space: $\U_j^{i} = \rho(g_j)$ for all $i$ and $j$. We run each of the three methods (naive, baseline, and \cref{alg:main}) for each group $G=C_n,D_n,S_n$, for tensors of order $d=2,3,4$ with $k=0$ (the output space is $\RR$), and $n=10,20,30,40,50$. In each configuration, a time limit of $1$ hour is enforced, and a main memory limit of $180$GB. The results are shown in \cref{fig:timecomp}. All methods, when they were not prematurely terminated, successfully extracted an orthonormal basis of the correct dimension, as determined by the naive method.
 
\begin{figure}[tb]
\caption{\footnotesize Construction times for group-invariant tensor bases.}
\label{fig:timecomp}

\centering 
\includegraphics[width=.95\textwidth]{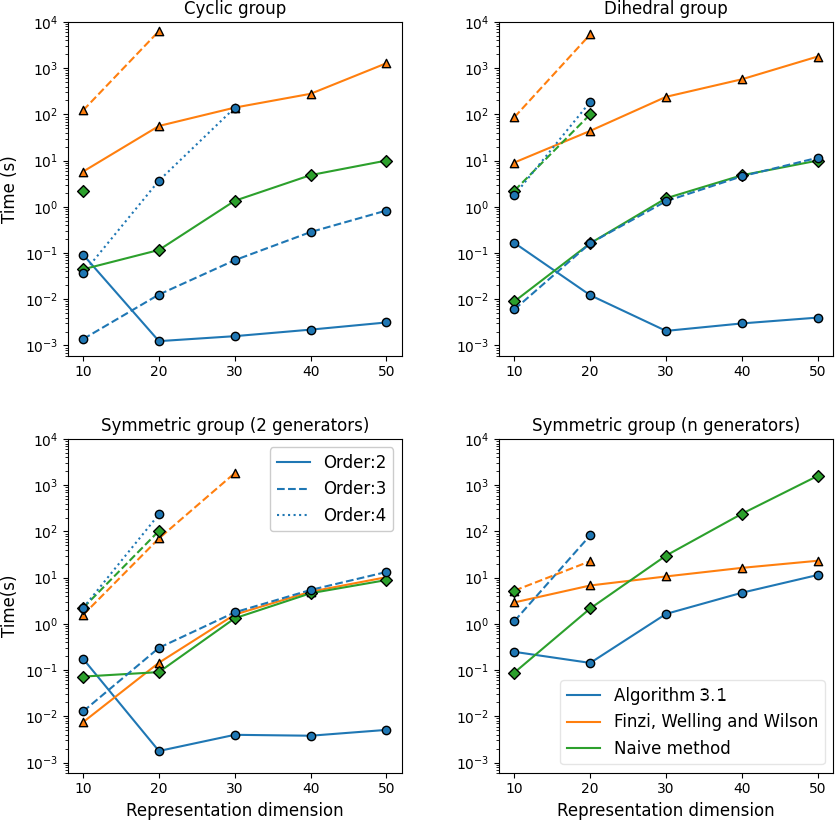}
\end{figure}

In all tested configurations, our method outperforms the baseline method and the naive method, often with multiple orders of magnitude, provided that the ``good'' set of generators $\{g_c, g_s^{12}\}$ is chosen for $S_n$.
The naive method outperforms the baseline in some small-scale cases as well because of the higher throughput realized by dense linear algebra kernels; Finzi, Welling, and Wilson's algorithm \cite{finzi2021practical}, on the other hand, was specifically developed for large-scale cases where the naive method is not feasible anymore due to excessive memory consumption (also visible in \cref{fig:timecomp}).

The bottom row of \cref{fig:timecomp} highlights the importance of selecting the right generators. While the naive method is competitive with \cite[Algorithm 1]{finzi2021practical} when viewing $S_n = \langle g_c, g_s^{12} \rangle$ (\cref{fig:timecomp}, bottom left), the former has the worst asymptotic scaling behaviour when using the alternative set of generators $\{g_s^{12}, g_s^{13},\ldots,g_s^{1n}\}$ instead (\cref{fig:timecomp}, bottom right). Similarly, \cref{alg:main} is also orders of magnitude slower with the latter ``bad'' selection of generators. Indeed, in this case, the initial filtering step can only reduce to a joint eigendecomposition problem of dimension $p=(n-2)^d \approx n^d$, whereas using the generators $\{ g_c, g_s^{12}\}$ reduces the eigenproblem to $p=n^{d-1}$. This means that with the two generators the cost of solving the eigenproblem, which dominates the time complexity, is $O(n^{3d -3})$, versus $O(n^{3d})$ when using the $n$ generators. This performance gap is consistent with the bottom row of \cref{fig:timecomp}.

\subsection{An invariant TTN for parity classification} \label{sec_sub_parity}
As a first toy example, we consider the problem of classifying bit-strings according to their parity, which has a group invariance.
The group action on a bit-string consists of replacing a $0$ by a $1$ and vice versa. For strings of odd length, if the parity is 1, this means there is an odd amount of ones and an even amount of zeros, and after the group action there is an even amount of ones. Thus, this problem is $G$-invariant for the parity group $G=(\ZZ_2, +_2)$. For strings of even length, if the parity is $1$, there is an odd amount of ones and an odd amount of zeros. Hence, after the group action, the parity remains the same, so the problem is $G$-invariant with the trivial representation on the output.

As local feature map $\phi$ that preserves $G$-invariance we use the one-hot encoding of the two group elements. That is $\phi(0)=\left(\begin{smallmatrix}1\\0\end{smallmatrix}\right)$ and $\phi(1)=\left(\begin{smallmatrix}0\\1\end{smallmatrix}\right)$. For the representation on the input vector spaces $\VS{W}_i = \RR^2$ and the output vector space $\VS{W}=\RR^2$ we use the representation defined in \cref{eq:rep_parity}, while extensions thereof are used as representations on the bond vector spaces $\VS{B}_i = \VS{B}_i^* = \RR^b$, where $b$ is the bond dimension. For even-length strings, a network that is $G$-invariant can be achieved by using the trivial representation on the output vector space. 

To train a model for classifying bitstrings of length $d=2\ell+1$, we construct a tensor train with $d$ cores as in \cref{sec_tt_explain}. The first and the last cores are order-$2$ tensors in $\RR^2 \otimes \RR^b$ and $\RR^b \otimes \RR^2$ respectively, where $b$ denotes the bond dimension. The output vector space is positioned at the $\ell$th core, which lives in $\RR^2 \otimes \RR^b \otimes \RR^b \otimes \RR^2$, the other tensors are order-$3$ tensors in $\RR^2 \otimes \RR^b \otimes \RR^b$.  As baseline model we use aforementioned TTN without imposing invariance, while $G$-invariance is imposed on the model referred to below as the $\ZZ_2$-invariant model. 

The training setup is as follows. 
\Cref{alg:main} constructs an orthonormal basis for a $G$-invariant core tensor. A general $G$-invariant core tensor $\tensor{T}$ is a linear combination of these basis vectors. Hence, optimization over the subspace of $G$-invariant core tensors can be reduced to unconstrained optimization over the coefficients with respect to this orthonormal basis.
The variational parameters of the complete $G$-invariant TTN are thus these sets of coefficients. 
Above two TTNs were trained with Tensorflow for fixed bit-string length over $100$ epochs with Adam optimization using automatic differentiation to compute the gradient. A softmax activation function is applied to normalize the outputs of both models. The core tensors in the baseline model are initialized such that the order-2 tensors are $\Id_{2\times b} := [\Id_{2\times 2} \; 0_{2\times (b-2)}]$, the elements in the output core are all $1$ and the other cores are initialized such that the matrix ``slices'' obtained by fixing the incoming bond dimension are $\Id$; that is, for a tensor $\tensor{T}$, for all $1\leq i \leq \dim b$, $T_{*,*,i} = \frac{1}{\sqrt{2}}\Id_{2\times b}$ or $T_{*,i,*} = \frac{1}{\sqrt{2}} \Id_{2\times b}$, depending on whether the core is respectively to the left or right of the output core. To all cores identically and independently distributed Gaussian noise with average $0$ and standard deviation $10^{-3}$ is added.
As training data, a random subset consisting of $5\%$ of all possible bitstrings were taken. \textit{No validation set was employed}. Instead we simply minimize the binary cross-entropy loss on the training data, \textit{facilitating overfitting on the data}. The goal is precisely to illustrate that imposing the $\ZZ_2$-invariance is a useful inductive bias and can prevent some of the overfitting.
The model performance was evaluated on the bitstrings that were not used for training, comprising our test set. No hyperparameter tuning was performed. We trained the models $100$ times with a different training set and different initialization of the TTN parameters.

\begin{figure}
\caption{\footnotesize Parity classification results (both on training and validation set) on $100$ runs after training for $100$ epochs.}
\label{fig:parity_comp}

\centering
\includegraphics[width=\textwidth]{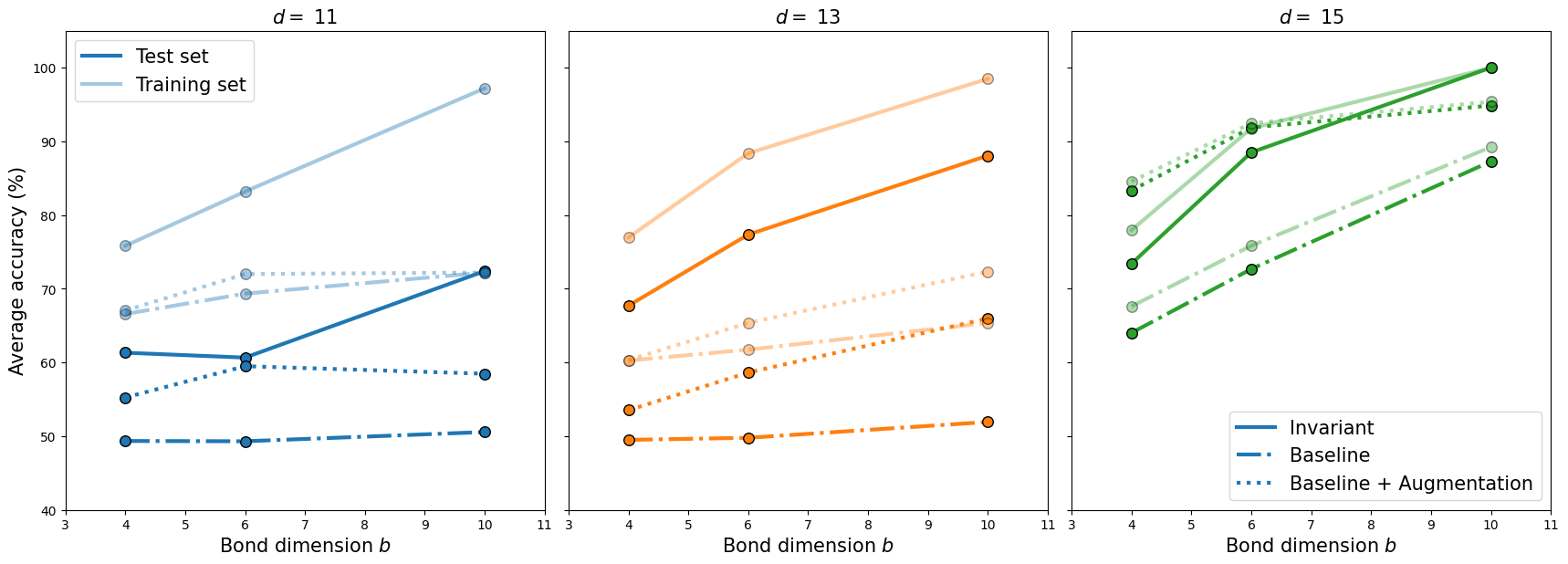}
\end{figure}

Results for odd-length bit-strings for different string lengths and bond dimensions are shown in \cref{fig:parity_comp}. 
The plots show the average accuracy over the 100 runs on both test and training set. We compared $\ZZ_2$-invariant TTN models, regular TTN models, and the regular TTN model trained with complete \textit{data-augmentation}; that is, after creating a training set, the $(\ZZ_2,+_2)$-transformed training set is added to the initial training set, thereby doubling the amount of training data.

Apart from the general conclusion regarding TTNs that increasing the bond dimension also increases the model performance, it is clear that the $\ZZ_2$-invariant model outperforms the baseline model in every case. Eventually, the invariant model outperforms the data-augmented model as well. Where this is not the case, the bond dimension is probably too small, limiting the model's expressiveness. The performance of all models appears to increase as the problem size increases, but this attributed to the increasing training dataset size, which is $0.05\cdot2^d$.

Apart from an increased performance, the invariant TTN has less free parameters (as determined by Tensorflow's \texttt{count\_params} function); this is shown in \cref{tab:model_size}. On average, in this example, the number of free, trainable parameters was reduced by about a factor $2$ by using $\ZZ_2$-invariant networks. Naturally, this depends entirely on the group, the tensor order, and dimensions. Memory consumption decreases as well, because there are less parameters and the basis tensors need only be stored a single time, and thus the reduction increases as the model grows.

\begin{table}[t] \footnotesize
	\caption{\footnotesize Number of variational parameters and total storage requirements in number of floating-point values of the two models. The ratios between those quantities of the baseline and invariant model are also shown.} \label{tab:model_size}
	\begin{center}
		\begin{tabular}{c|lccccccccc}
			\toprule
			 & & \multicolumn{3}{c}{$d=11$} & \multicolumn{3}{c}{$d=13$} & \multicolumn{3}{c}{$d=15$}\\
			\cmidrule{3-11}
			& $b$ & 4 & 6 & 10 & 4 & 6 & 10 & 4 & 6 & 10 \\
			\midrule
			\multirow{3}{5em}{Trainable parameters (floats)} 
			& Baseline   & $338$ & $746$ & $2042$ & $402$ & $890$ & $2442$ & $466$ & $1034$ & $2842$ \\
			& Invariant  & $168$ & $372$ & $1020$ & $200$ & $444$ & $1220$ & $232$ & $516$ & $1420$ \\
			\cmidrule{2-11}
			& Ratio & $2.01$ & $2.01$ & $2.00$ & $2.01$ & $2.00$ & $2.00$ & $2.01$ & $2.00$ & $2.00$ \\
			\midrule
			\multirow{3}{5em}{Storage (floats)} 
			& Baseline  & $338$ & $746$ & $2042$ & $402$ & $890$ & $2442$ & $466$ & $1034$ & $2842$ \\
			& Invariant & $280$ & $612$ & $1660$ & $312$ & $684$ & $1860$ & $344$ & $756$ & $2060$ \\
			\cmidrule{2-11}
			& Ratio & 1.21 & 1.22 &  1.23 & 1.29 & 1.30 & 1.31 & 1.35 & 1.37 & 1.38 \\
			\bottomrule	
		\end{tabular}
	\end{center}
\end{table}

\subsection{Invariant TTNs for transcription factor binding} \label{sec_suplearn_tt_expers}\label{sec:RC}
We illustrate the power of the introduced invariant TTNs by applying them to a binary classification task on DNA sequences. The task is to predict whether a \textit{transcription factor} (a protein) will bind to a DNA sequence. The data set of Zhou, Shrikumar and Kundaje \cite{Zhou2020.11.04.368803} contains three transcription factors (MAX, CTCF and SPI1) along with $10\,000$ DNA strands per transcription factor. Each DNA strand in this dataset is composed of a sequence of $1\,000$ canonical base pairs, i.e., adenine (A), cytosine (C), guanine (G), and thymine (T). The authors of \cite{Zhou2020.11.04.368803} randomly partitioned the data, allocating $40\%$ to the training set, $30\%$ to the test set, and $30\%$ to the validation set.

When used for predicting genome-wide regulatory signals such as transcription factor binding, DNA strands exhibit an interesting invariance called \textit{reverse compliment (RC) invariance} \cite{Zhou2020.11.04.368803}. The complement invariance arises from the nucleobase pairings in the double helix (A pairs up with T, and G pairs up with C) and the way the strands are read. The reverse invariance is geometric in nature: if a transcription factor binds to a DNA strand, then it also binds on the same strand rotated by $\pi$ radians by rotating the protein likewise. For example, AGTGC is equivalent to GCACT and if a transcription factor binds to one of them it will bind to the other as well. 

RC invariance has both a local (the complement invariance) and a global part (the reverse invariance). This paper was concerned with local invariance. However, this particular global invariance can also be handled with our invariant TTNs. 

The architecture of our TTN is as follows. As nucleobase feature maps $\phi$, binary vectors of length $4$ that represent a one-hot encoding of the molecules A, G, T, and C in this respective order are chosen. 
A TTN with $d=1001$ cores is constructed as a classification model. The output is located on the core in the middle, which does not have an input vector space and thus lives in $\RR^b\otimes \RR^b\otimes\RR^2$. The output space is binary because a strand binds to the protein or it does not.
The local part is described by the cyclic group $(\langle g \rangle, \circ)$, where the action is such that $g(\text{A})=\text{C}$, $g(\text{C})=\text{A}$, $g(\text{G})=\text{T}$ and $g(\text{T})=\text{G}$; note that it is isomorphic to the parity group $G=(\ZZ_2, +_2)$. This naturally induces a $4$-dimensional representation $\U_g = \U_r$ that can be used on the input vector spaces $\VS{W}_i = \RR^4$, where $\U_r$ is as in \cref{eqn_Uc_and_Ur}. 
For the bond vector spaces $\VS{B}_i$, we take as representation of $G$ the reverser matrix $\U_r$ in dimension $\dim \VS{B}_i$. To impose RC-invariance on the TTN we do not require the building blocks to be $G$-invariant as before, but we require the cores on the right of the middle core to be related to their corresponding cores on the left side; that is, for strands of length $d$ the core at position $m$ is related to the core at position $d-m+1$ in the TTN in a way such that transposing the bond indices of the core has the same result as applying the invariance action for the $(\ZZ_2,+_2)$-part on that core. Only the output core (at position $\ell=\frac{d+1}{2}=501$) is $(\ZZ_2,+_2)$-invariant, with an extra transposition for the reverse part of the invariance. This is summarized by the following conditions, using the TTN notation from \cref{sec_tt_explain}:
\begin{align*} 
	\tensor{A}_{ijl}^{\ell} &= \left((\U_r, \U_r, \Id)\cdot\tensor{A}^{\ell}\right)_{jil}, \\
	\tensor{A}^{m}_{ijk} &= \left((\U_r, \U_r, \U_r)\cdot\vb{\mathcal{A}}^{d-m+1}\right)_{ikj}, \quad m=1,2,\ldots,\ell-1, \\
	\A^{d}_{ij}  &= \left((\U_r, \U_r)\cdot\A^{1}\right)_{ji};
\end{align*} 
herein, we used that the reverser matrix is symmetric $U_g = U_r = U_r^\top$.
An RC-action on this TTN then swaps all cores at locations $m$ with the core at at position $d-m+1$ while simultaneously applying the $(\ZZ_2,+_2)$ action on all input vector spaces and transposing the bond vector spaces.
For odd-length strands we can make the middle core have an input vector space $\VS{W}^*$ as well, making the core live in $\RR^2 \otimes \RR^b\otimes\RR^b \otimes \RR^2$. An example of an RC-invariant TTN together with the conditions on the building blocks is given in \cref{fig:rc_network}; it can be verified that under the given conditions the TTN is invariant. 

Training was done for $100$ epochs and a batch size of $100$ examples, with binary cross-entropy loss and $\ell_2$-regularization on the variational parameters. The model is trained with stochastic gradient descent with Nesterov momentum with a fraction of $0.2$, with different learning rates (see \cref{tab:rc_params}). The TTN output is again normalized with a \textit{softmax} activation function. After every epoch, the \textit{area under the ROC-curve} (AUROC) on the validation set is evaluated and the weights of the best performing model are saved and used for evaluating the model performance on the test set. Optimal hyperparameters vary depending on the prediction task and are given in \cref{tab:rc_params}; they were found by manually trying out a few values. 

\begin{table}[tb]
\footnotesize
		\caption{\footnotesize Optimal training parameters used for evaluating model.}\label{tab:rc_params}
		\begin{center}
			\begin{tabular}{ccccc} \toprule
				\bf Task & \bf Bond dimension & \bf Regularization & \bf Epochs & \bf Learning rate \\ \midrule
				MAX & 3 & 0.005 & 100 & 0.001 \\
				CTCF & 8 & 0.005 & 100 & 0.01 \\
				SPI1 & 8 & 0.003 & 100 & 0.01 \\ \bottomrule
			\end{tabular}
		\end{center}
\end{table}

Average results over 5 runs of our model together with the results from the state-of-the-art invariant neural network model introduced by Mallet and Vert \cite{mallet2021reversecomplement} are given in \cref{tab:res_rc}. Their model takes an additional translation invariance into account. The performance of our TTN model is competitive with the state of the art for the CTCF and SPI1 dataset, while outperforming it for the MAX dataset. However, we note that for about one in five runs, our TTN model starts in a local minimum and does not get to a loss value that is much better than the initial value. These cases are left out of the averages and variances, and may have be avoided by using optimization algorithms that are better suited for training TTNs such as the one introduced by Stoudenmire and Schwab in \cite{stoudenmire2017supervised} or by better learning-hyperparameter finetuning.

\begin{figure}[tb]
	\caption{\footnotesize RC-invariant TTN architecture and building block constraints in Penrose graphical notation \cite{Bridgeman_2017}. The direction of the arrows indicates whether the vector space is a dual space or not. The constraints are invariance constraints as introduced earlier with an extra transposition of the bond indices to account for the reverse operation of the symmetry. To arrive at a model that is RC-invariant, the trivial representation is taken on the output vector space.}
	\label{fig:rc_network}
	\centering
\begin{tikzpicture}[scale=0.75]
\draw[very thick,->,>=stealth] (3,1) -- ++(0,1.2);
\foreach \i in {0,1,2,3,4} {
\ifthenelse{\i = 2}{}{\draw[very thick,->,>=stealth] (1.5*\i,0) -- ++(0,0.6);}
\draw[very thick,fill=white,rounded corners] (1.5*\i-.5,0.6) rectangle ++(1,1);
\ifthenelse{\i < 2}{
\draw[very thick,->,>=stealth] (1.5*\i+0.5,1.1) -- ++(0.5,0);
}{}
\ifthenelse{\i > 2}{
\draw[very thick,->,>=stealth] (1.5*\i-0.5,1.1) -- ++(-0.5,0);
}
}
\node at (0,1.1) {$A^1$};
\node at (1.5,1.1) {$\tensor{A}^2$};
\node at (3,1.1) {$\tensor{A}^3$};
\node at (4.5,1.1) {$\tensor{A}^4$};
\node at (6,1.1) {$A^5$};
\node at (7.5,1.1) {with};
\begin{scope}[shift={(10,4.1)},scale=.75]
\begin{scope}[shift={(0,0)}]
\draw[very thick,fill=white,rounded corners] (-1.5,-0.5) rectangle ++(1,1);
\node at (-1,0) {$A^5$};
\draw[very thick,->,>=stealth] (-1.5,0) -- ++(-.6,0) node[above]{$j$};
\draw[very thick,->,>=stealth] (-1,-1.1) node[left]{$i$} -- ++(0,.6);
\end{scope}
\node at (1,0) {$=$};
\begin{scope}[shift={(2.5,0)}]
\draw[very thick,fill=white,rounded corners] (1.5,-0.5) rectangle ++(1,1);
\node at (2,0) {$A^1$};
\draw[very thick,->,>=stealth] (2,-1.5) -- ++(0,1);
\draw[very thick,->,>=stealth] (2,-2.5) node[left]{$i$} -- ++(0,0.6);
\node[draw,rectangle,thick,rounded corners,fill=black!8] at (2,-1.5) {$U_r$};
\draw[very thick,->,>=stealth] (2.5,0) -- ++(0.5,0);
\draw[very thick,->,>=stealth] (4.0,0) -- ++(0.5,0) node[above]{$j$}; 
\node[draw,rectangle,thick,rounded corners,fill=black!8] at (3.5,0) {$U_r$};
\end{scope}
\end{scope}
\begin{scope}[shift={(10,1.1)},scale=.75]
\begin{scope}[shift={(0,0)}]
\draw[very thick,fill=white,rounded corners] (-1.5,-0.5) rectangle ++(1,1);
\node at (-1,0) {$\tensor{A}^4$};
\draw[very thick,->,>=stealth] (-1.5,0) -- ++(-.6,0) node[above]{$j$};
\draw[very thick,->,>=stealth] (0.1,0) node[above]{$i$} -- ++(-.6,0);
\draw[very thick,->,>=stealth] (-1,-1.1) node[left]{$k$} -- ++(0,.6);
\end{scope}
\node at (1,0) {$=$};
\begin{scope}[shift={(2.5,0)}]
\draw[very thick,fill=white,rounded corners] (1.5,-0.5) rectangle ++(1,1);
\node at (2,0) {$\tensor{A}^2$};
\draw[very thick,->,>=stealth] (2,-1.5) -- ++(0,1);
\draw[very thick,->,>=stealth] (-0.5,0) node[above]{$i$} -- ++(0.5,0);
\draw[very thick,->,>=stealth] (1.0,0) -- ++(0.5,0); 
\node[draw,rectangle,thick,rounded corners,fill=black!8] at (0.5,0) {$U_r$};
\draw[very thick,->,>=stealth] (2,-2.5) node[left]{$k$} -- ++(0,0.6);
\node[draw,rectangle,thick,rounded corners,fill=black!8] at (2,-1.5) {$U_r$};
\draw[very thick,->,>=stealth] (2.5,0) -- ++(0.5,0);
\draw[very thick,->,>=stealth] (4.0,0) -- ++(0.5,0) node[above]{$j$}; 
\node[draw,rectangle,thick,rounded corners,fill=black!8] at (3.5,0) {$U_r$};
\end{scope}
\end{scope}
\begin{scope}[shift={(10,-2.4)},scale=.8]
\begin{scope}[shift={(0,0)}]
\draw[very thick,fill=white,rounded corners] (-1.5,-0.5) rectangle ++(1,1);
\node at (-1,0) {$\tensor{A}^3$};
\draw[very thick,->,>=stealth] (-2.1,0) node[above]{$j$} -- ++(.6,0);
\draw[very thick,->,>=stealth] (0.1,0) node[above]{$i$} -- ++(-.6,0);
\draw[very thick,->,>=stealth] (-1,0.5) -- ++(0,.6) node[left]{$k$};
\end{scope}
\node at (1,0) {$=$};
\begin{scope}[shift={(2.5,0)}]
\draw[very thick,fill=white,rounded corners] (1.5,-0.5) rectangle ++(1,1);
\node at (2,0) {$\tensor{A}^3$};
\draw[very thick,->,>=stealth] (2,0.5) -- ++(0,0.6) node[left]{$k$};
\draw[very thick,->,>=stealth] (-0.5,0) node[above]{$i$} -- ++(0.5,0);
\draw[very thick,->,>=stealth] (1.0,0) -- ++(0.5,0); 
\node[draw,rectangle,thick,rounded corners,fill=black!8] at (0.5,0) {$U_r$};
\draw[very thick,->,>=stealth] (3.0,0) -- ++(-0.5,0);
\draw[very thick,->,>=stealth] (4.5,0) node[above]{$j$} -- ++(-0.5,0); 
\node[draw,rectangle,thick,rounded corners,fill=black!8] at (3.5,0) {$U_r$};
\end{scope}
\end{scope}
\end{tikzpicture}
\end{figure}
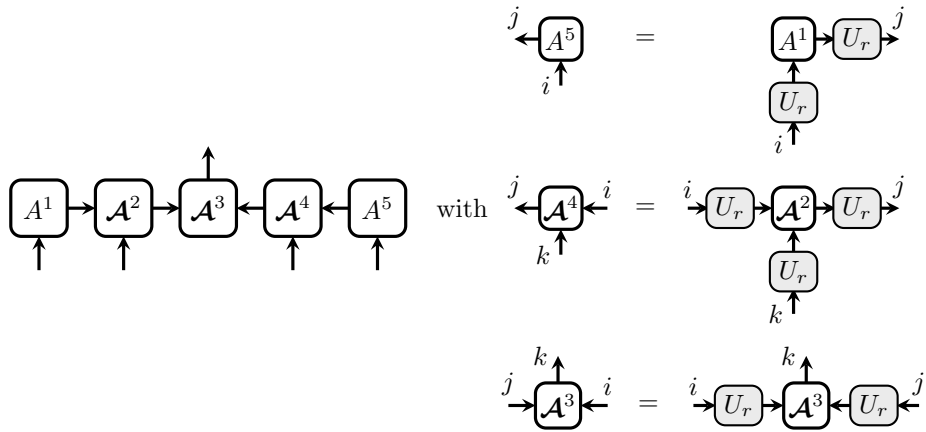

\begin{table}[tb]
\footnotesize
		\caption{\footnotesize Test results of RC-invariant TTNs and the benchmark results from Mallet and Vert \cite{mallet2021reversecomplement}.}\label{tab:res_rc}
		\begin{center}
			\begin{tabular}{cccc} 
			\toprule
				\bf Dataset & \bf Model & \bf AUROC & \bf Standard deviation \\ \midrule
				\multirow{2}{4em}{CTCF} & Ours & $94.10\%$ & $0.21\%$ \\
				& Benchmark & $\bm{98.84}\%$ & $0.056\%$ \\ \midrule
				\multirow{2}{4em}{SPI1} & Ours & $96.53\%$ & $0.030\%$ \\
				& Bechmark & $\bm{99.26}\%$ & $0.034\%$ \\ \midrule
				\multirow{2}{4em}{MAX} & Ours & $\bm{97.06}\%$ & $0.011\%$ \\
				& Benchmark & $92.80\%$ & $0.26\%$\\
				\bottomrule
			\end{tabular}
		\end{center}
\end{table}

\section{Conclusions}\label{sec:conclusions}
Group-invariant tensors arising from equivariant multilinear maps were introduced in \cref{sec:background}.
We presented a new method to construct $G$-invariant TTNs for arbitrary discrete groups $G$ with normal representations in \cref{sec:tnequiv}. The main ingredient is \cref{alg:main}, a basis construction method that scales better in practice than the state-of-the-art method  \cite[Algorithm 1]{finzi2021practical} for several common groups. It exploits an observation about real rightmost eigenvalues of a joint eigenproblem in \cref{prop_reduction}, allowing a reduction to a single standard eigenproblem. Crucial to the performance of \cref{alg:main} is the selection of the first generator of the group $G$. \Cref{sec:pval} suggests a good first generator for any mixed product of cyclic, dihedral, symmetric, and dicyclic groups. Our $G$-invariant TTNs were applied as supervised learning models for the prediction of transcription factor binding on DNA sequences in \cref{sec:experiments}. The group structure inherent in this problem (reverse complement invariance) is captured by the proposed invariant TTN, hereby introducing a powerful and application-supported inductive bias. The experiments show that TTNs can be competitive with state-of-the-art neural network models on a real world problem in terms of accuracy.


\end{document}